\documentclass{article}
\usepackage{appendix}
\usepackage{comment}
\usepackage{algpseudocode}

\PassOptionsToPackage{numbers, compress}{natbib}

\usepackage{subcaption}

\usepackage[preprint]{neurips_2023}

\usepackage{amsthm,amsmath}
\newtheorem{theorem}{Theorem}
\usepackage{graphicx}
\usepackage[ruled,vlined]{algorithm2e}
\usepackage[utf8]{inputenc} 
\usepackage[T1]{fontenc}    
\usepackage{hyperref}       
\usepackage{url}            
\usepackage{booktabs}       
\usepackage{amsfonts}       
\usepackage{nicefrac}       
\usepackage{microtype}      
\usepackage{xcolor}         

\newtheorem{proposition}{Proposition}[section]
\title{Beyond Dynamic Programming}

\author{%
 Abhinav Muraleedharan\thanks{Graduate Student, University of Toronto.} \\
  University of Toronto\\
  Toronto,
 Canada \\
  \texttt{Abhinav.Muraleedharan@mail.utoronto.ca} \\
}
\begin{document}
\maketitle
\begin{abstract}
 In this paper, we present Score-life programming, a novel theoretical approach for solving reinforcement learning problems. In contrast with classical dynamic programming-based methods, our method can search over non-stationary policy functions, and can directly compute optimal infinite horizon action sequences from a given state. The central idea in our method is the construction of a mapping between infinite horizon action sequences and real numbers in a bounded interval. This construction enables us to formulate an optimization problem for directly computing optimal infinite horizon action sequences, without requiring a policy function. We demonstrate the effectiveness of our approach by applying it to nonlinear optimal control problems. Overall, our contributions provide a novel theoretical framework for formulating and solving reinforcement learning problems.
\end{abstract}

\section{Introduction}
Reinforcement learning \cite{sutton2018reinforcement} is a principled mathematical framework for designing autonomous agents that can interact with the environment and produce optimal behaviours. Notable accomplishments in demanding AI problem spaces have been achieved through the utilization of reinforcement learning methodologies. Examples include surpassing human-level performance in the game of Go \cite{silver2016mastering}, facilitating intricate manipulation abilities in robotic agents \cite{andrychowicz2020learning}, and addressing a wide array of challenging AI problems \cite{li2017deep}. At the heart of these reinforcement learning algorithms lies the Dynamic Programming method, first outlined by Bellman \cite{bellman1966dynamic}. In Dynamic programming-based methods, the design of intelligent agents is reduced to the mathematical problem of computing an optimal policy function, to minimize a specific cumulative cost. Many reinforcement learning algorithms such as Temporal Difference learning \cite{sutton1988learning}, Fitted Value Iteration \cite{gordon1995stable,ernst2005tree,munos2008finite}, and Deep Q Network \cite{mnih2015human} are based on the classical Dynamic Programming method. Classical Dynamic programming-based methods are however mostly limited to stationary policies (time-invariant policy functions). However, several examples exist where non-stationary policies are of interest \cite{bertsekas2022abstract}. \\

Consider an agent $A$ living in a bounded state space $X$ containing $P$ different states. Let the the set of all possible  actions be $U$ and let $|U|=M$. From any given state $x$, the agent can apply actions sequentially, from the set $U$. The set of all possible infinite horizon action sequences, taken from the state $x$ is uncountably infinite. However, the set of all infinite horizon action sequences that can be generated by a Dynamic Programming approach, from a deterministic feedback policy $\pi(x)$ is finite and is equal to $P^M$. Hence, a Dynamic Programming-based method would not be able to search over all possible infinite horizon action sequences. This would limit the set of all possible behaviors that can be achieved in an intelligent agent, and hence the development of methods for computing optimal non-stationary policies is of significant theoretical and practical interest.
\\

In this paper, we introduce an alternative theoretical approach to solving the infinite horizon problem, that can search over non-stationary and stationary policies in a discrete-time setting. We also do not assume monotonicity and contraction properties of the Bellman Operator \cite{bertsekas2022abstract}, and hence our method is applicable to scenarios when the Bellman Operator is not a contraction. Mathematically, our approach involves the construction of a mapping between real numbers and infinite horizon action sequences. Each infinite horizon action sequence is mapped to a unique real number in the interval \begin{math}
    [0,1)
\end{math}. Next, we define a function that maps the real numbered values to infinite horizon costs of specific action sequences. We then formulate an optimization problem to compute the optimal cost to go and optimal infinite horizon action sequence. Our approach also bridges fractal functions and problems in optimal control/RL, and offers novel theoretical insights for solving nonlinear optimal control problems.
\section{Problem Setting}
\noindent We consider the deterministic, discrete-time RL problem. The system dynamics can be expressed as:
\begin{equation}
   {x_{k+1}} = f({x_{k}},u_k) 
\end{equation}
where \(x_k \in X\) is the state at timestep \(k\) and \(u_k \in U \) is the control input timestep \(k\). The set \(X\) is a countably infinite set consisting of all possible states in the state space. 
The infinite horizon cost function can be expressed as:
\begin{equation}
  J_{\infty}(x_0) =  \sum_{k=0}^{\infty} \gamma^k.{g(x_k,u_k)}
\end{equation}
where \(\gamma \in (0,1)\) and \( \gamma^k g(x_k,u_k)\) is the stage cost at time step $k$. We further assume that $u_k \in U = \{a_0, a_1, a_k,...a_{M-1}\}$ and $U$ is a finite set containing $M$ different action values. Given a state \begin{math}x \in X\end{math}, our goal is to compute the infinite sequence of actions \begin{math} \{u\}_{k=0}^{\infty}\end{math} that minimizes the infinite horizon cost defined in eq(2). The standard approach here is to define a stationary policy function \begin{math}\pi(x)\end{math} which is a mapping from the state space to action space, and compute the optimal policy function by applying Bellman operators. Contrary to this, we proceed to directly define an optimization problem over the infinite sequence of actions, and minimize the infinite horizon cost without the construction of an explicit policy function.

\section{Theory}
In this section, we introduce the theory behind the Score-life programming approach. This section is organized into three parts. First, we explain how infinite horizon action sequences can be represented as real numbers in a bounded interval. Next, we show how the real numbers representing action sequences can be computed from a deterministic policy function. Finally, we introduce the Score-life function, which is a mapping from the bounded interval representing action sequences to infinite horizon cost corresponding to different action sequences. We derive important properties of the Score-life function and explain how the Score-life functions of different states are interrelated. 
\subsection{Representing action sequences as real numbers}

Consider an agent living in an environment with dynamics \begin{math}x_{k+1} = f(x_k)\end{math} taking a sequence of actions \begin{math} \{u_k\}_{k=0}^{\infty}\end{math} from the action set \begin{math}U\end{math}. Each infinite action sequence forms an action trajectory \begin{math}A_x = \{u_0\}\{u_1\}\{u_2\}.....\{u_{\infty}\}\end{math}. Let \begin{math}A\end{math} be the set of all possible infinite horizon action trajectories. For any action set \begin{math}U\end{math} with \begin{math}|U| > 1\end{math}, the set \begin{math}A\end{math} is uncountably infinite. The set of all possible action sequences does not have an implicit structure, hence we define some structure on this set for formulating an optimization problem directly over action sequences. Specifically, we equip the set \begin{math}A\end{math} with a topology \begin{math}\mathcal{O}\end{math} by defining the following encoding scheme. We map action values \begin{math}a_k\end{math} to binary digit sequence of length \begin{math}log(M)\end{math} where \begin{math}M\end{math} is the cardinality of the action set \begin{math}U\end{math}. Specifically, we define a surjective map \begin{math}\kappa\end{math} from the action set \begin{math}U\end{math} to the set of binary sequences of length \begin{math}log(M)\end{math}.\\
For instance, consider an agent with action set \begin{math}U\end{math} to the set of binary sequences \begin{math}U = \{-10,10\}\end{math}. In this case, we can map \begin{math}a_0 = -10 \end{math} to 0 and \begin{math}a_1 = 10\end{math} to 1.
\[\kappa(-10) \rightarrow 0\]
\[\kappa(10) \rightarrow 1\]
Now, an infinite sequence of actions can be mapped to an infinite binary sequence, with \begin{math}k^{th}\end{math} element of the binary sequence mapping to action taken at timestep \begin{math}k\end{math}. The binary sequence can then be mapped onto a bounded one-dimensional interval.

The infinite sequence of actions taken by an agent be called 'life' and the real number corresponding to to the infinite sequences be denoted by $l$.\\
When $|U| = 2$, the life value is given by:
\begin{equation}
    l = \sum_{i=0}^{\infty}2^{-i-1}.\kappa(u_i) 
\end{equation}
Where \begin{math}u_i\in \end{math} denotes control input at time step \begin{math}i \end{math}.
For general case where $|U| = M$, the life value is given by:
\begin{equation}
    l = \sum_{i=0}^{\infty}2^{logM.(-i-1)}.\kappa(u_i)
\end{equation}
Representing action sequences in this form has several advantages. First, for any finite \begin{math} |U|\end{math}, infinite action sequence of the agent can be mapped to a bounded interval between 0 and 1. Secondly, we can derive recursive relations that relates action sequences taken from different states. Specifically, let \begin{math}l_0 \end{math} be the life value of the action sequence taken by an agent at state \begin{math}x_0\end{math}. At time step 1, the agent reaches state \begin{math}x_1 = f(x_0,u_0) \end{math}. Let \begin{math}l_1\in \end{math} be the infinite horizon action sequence taken by the agent in state \begin{math}x_1 \end{math}. Now, \begin{math}l_1 \end{math} can be obtained by performing a multi-bit shift operation on \begin{math}l_0\in \end{math}. Mathematically, the both life values are related by the following equation: 
\[l_1 = \{2^{logM}l_0\}\]
where \begin{math} \{\} \end{math} is the fractional part function. See Appendix A for derivation. 
\subsection{Computation of life values for deterministic policies}
In this section, we explain how life values can be derived from deterministic policy functions. Our insight here is that life values follow a linear recursive relation, and hence they can be computed by solving a set of linear equations. Consider an agent following a deterministic policy $\pi$ and let the environment dynamics be \begin{math}
    x_{k+1} = f(x_k,\pi(x_k))
\end{math}. Let the cardinality of action space be equal to \begin{math}
    M
\end{math}. 
Consider the state \begin{math}
    x_0
\end{math}. The life value corresponding to the infinite action sequence starting from the \begin{math}
    x_0
\end{math} is given by:
\begin{equation}
l_{\pi}(x_0) = 2^{-log M}.\kappa(\pi(x_0)) + 2^{-2 logM}. \kappa(\pi(x_1)) + ..
\end{equation}
Eq(21) can be written in a recursive form as shown below:
\begin{equation}
l_{\pi}(x_0) = 2^{-log M}. \kappa(\pi(x_0)) +  2^{-log M}.l_{\pi}(f(x_0,\pi(x_0))
\end{equation}
For a finite state space with \begin{math}
    N
\end{math} number of states, we obtain \begin{math}
    N
\end{math} linear equations corresponding to life-values at \begin{math}
    N
\end{math} different states. Specifically, they can be written down as:
\begin{equation}
l_{\pi}(x_0) = 2^{-log M}. \kappa(\pi(x_0)) +  2^{-log M}.l_{\pi}(f(x_0,\pi(x_0))
\end{equation}
\begin{equation}
l_{\pi}(x_1) = 2^{-log M}. \kappa(\pi(x_1)) +  2^{-log M}.l_{\pi}(f(x_1,\pi(x_1))
\end{equation}

\begin{equation}
l_{\pi}(x_k) = 2^{-log M}. \kappa(\pi(x_k)) +  2^{-log M}.l_{\pi}(f(x_k,\pi(x_k))
\end{equation}

\begin{equation}
l_{\pi}(x_{N_s}) = 2^{-log M}. \kappa(\pi(x_{N_s})) +  2^{-log M}.l_{\pi}(f(x_{N_s},\pi(x_{N_s}))
\end{equation}
The set of linear equations can be expressed in matrix format as follows:

\begin{equation}
L_{\pi} = A.L_{\pi} + C_{\pi}
\end{equation}

where,
\begin{equation}
  L_{\pi} = [l_{\pi}(x_0), l_{\pi}(x_1), ....l_{\pi}(x_{N_s})]^T
\end{equation}
\begin{equation}
  C_{\pi} = (2^{-log M})[\kappa(\pi(x_0)), \kappa(\pi(x_1)), ....\kappa(\pi(x_{N_s}))]^T
\end{equation}
 and \begin{equation}
  A_{ij} = 2^{-log M} 
\end{equation} if \begin{math}
    x_j = f(x_i,\pi(x_i))
\end{math}.
 Else, $A_{ij}   = 0$
If \begin{math}
    (I - A)^{-1}
\end{math} exists, then the set of linear equations can be solved by computing: 
\begin{equation}
L_{\pi} = (I - A)^{-1}C_{\pi}
\end{equation} 
It is worth noting that this approach bears a strong resemblance to the computation of value functions through the Bellman update. This similarity arises from the fact that both the value function and the life values adhere to a recursive relationship.
\subsection{Infinite Horizon Cost and the Score-Life function}

Consider an agent that takes an infinite sequence of actions \begin{math}\{u_k\}_0^{\infty}\end{math} in an environment with dynamics \begin{math}x_{k+1} = f(x_k,u_k)\end{math}. Let the corresponding infinite horizon state sequence be \begin{math}\{x_k\}_0^{\infty}\end{math}.\\

The infinite horizon cost for the agent is given by:
\begin{equation}
  J_{\infty}(x_0) =  \sum_{k=0}^{\infty} \gamma^k.{g(x_k,u_k)}
\end{equation}
For a system with deterministic dynamics, the infinite state sequence  \begin{math}\{x_k\}_0^{\infty}\end{math} is directly dependent on the infinite action sequence \begin{math}\{u_k\}_0^{\infty}\end{math}. The infinite action sequence can be mapped onto a particular life value \begin{math}l\in [0,1)\end{math}.
\begin{equation}
  \{u_k\}_0^{\infty} \longrightarrow l  
\end{equation}
\begin{equation}
  S(l,x_0) = J_{\infty}(x_0) =  \sum_{k=0}^{\infty} \gamma^k.{g(x_k,u_k)}
\end{equation}
Now we can define a new function called 'Score-life' function, which maps the life value 'l' to its corresponding infinite horizon cost \begin{math}J_{\infty}(x_0)\end{math}.
For a fixed initial state $x_0$, the Score-life function: \begin{math}S(l.x_0):l \in [0,1) \rightarrow{S(l,x_0)}\end{math} maps infinite action sequences taken from state $x_0$ to respective infinite horizon cost values. We now introduce Theorem 1, which gives a recursive relationship between Score-life functions of different states. See Appendix A.1 for the proof.\\
\begin{theorem}
Let the dynamics of the environment satisfy the equation: \begin{math} x_{k+1} = f(x_k,u_k)\end{math}, where \begin{math}u_k \in U,x_k \in X, |U| = M  \end{math}. let the infinite horizon cost be given by \begin{math} J_{\infty}(x_0) =  \sum_{k=0}^{\infty} \gamma^k.{g(x_k,u_k)} \end{math}, and let \begin{math} \kappa\end{math} be the mapping from set \begin{math}U\end{math} to binary digit sequences. Then the Score-life function of the dynamical system obeys the following recursive equation:\\
\[S(l,x) = g(x,\kappa^{-1} (\lfloor 2^{log M}l\rfloor) ) + \gamma S(\{2^{logM}l\},f(x, \kappa^{-1} (\lfloor 2^{log M}l\rfloor))) \]

\end{theorem}
\begin{figure}[ht]
\begin{center}
\includegraphics[scale=0.3]{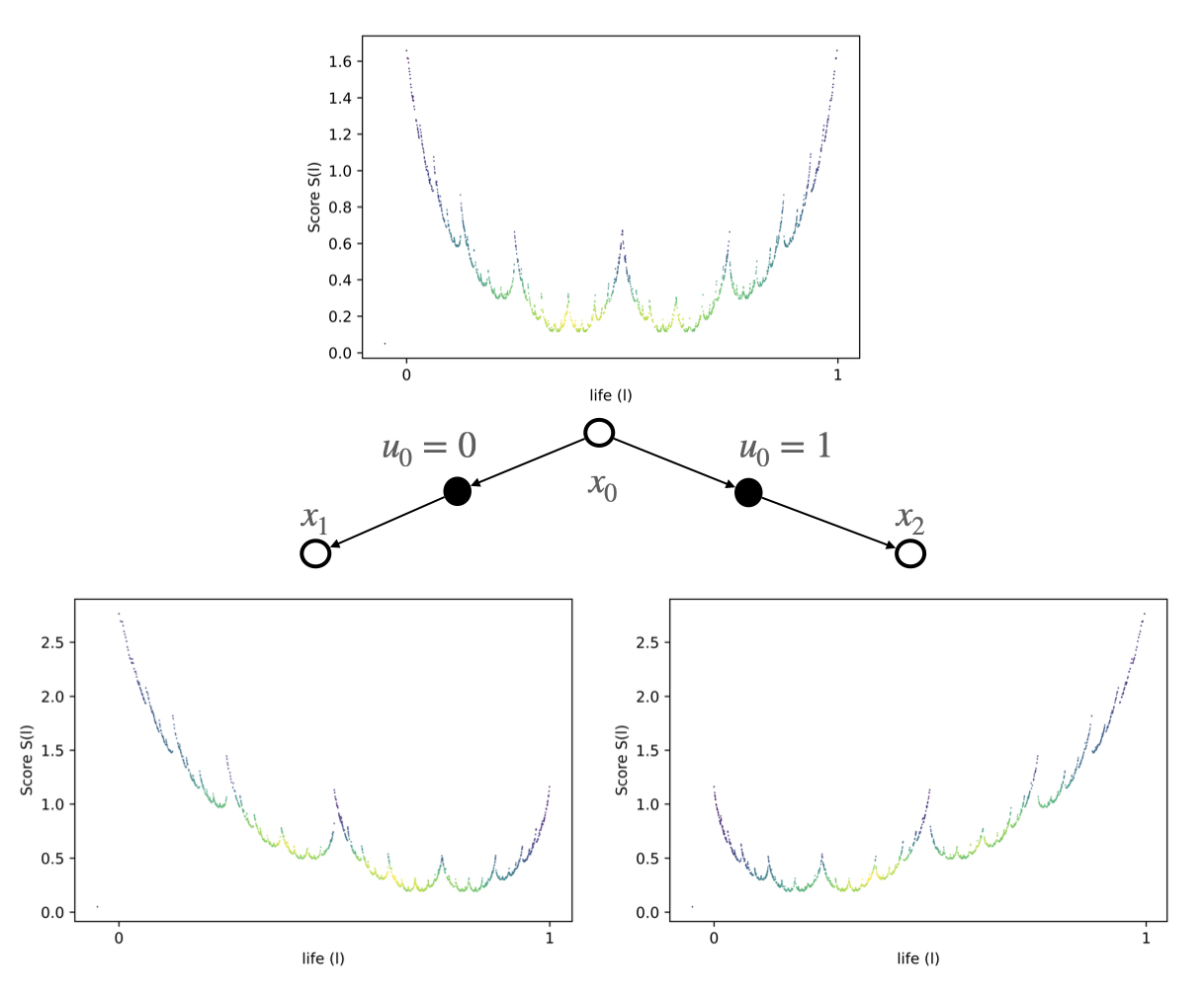}
\end{center}
\caption{Score life function of the origin state \begin{math}\textbf{x} = [x,\dot{x},\theta,\dot{\theta}] = [0,0,0,0]\end{math} of the carpole dynamical system (top) and Score-life functions of neighbouring states (bottom left and bottom right). The stage cost function is equal to \begin{math}
    x^TQx
\end{math} and \begin{math}
    \gamma = 0.5
\end{math}. }
\centering
\end{figure}
Theorem 1 tells us that the Score-life function of a given state is given by the superposition of Score-life function of neighbouring states. This implies that the Score-life function is a fractal function (See Fig 1), and the Score-life function of a single state contains information about Score-life functions of all the other states in the state-space.
The Score-life function for all states can be computed by applying the recursive update equation in Theorem 1.

\begin{algorithm}[H]
\caption{Recursive Computation of Score-life Value Function}
\SetAlgoLined

\textbf{Input: }
 Dynamics Model \begin{math}x_{k+1} = f(x_k,u_k)\end{math}, Finite region of state space \begin{math}X_f \subseteq X \end{math}\\
   \textbf{Result: }
 Score-life value function \begin{math}S(l,x)\end{math} for all states \begin{math}x \in X_f\end{math}\\
 
\textbf{Initialize: } \begin{math}S(l,x) = 0\end{math}, \begin{math}\forall x \in X_f\end{math}, \begin{math}l \in[0,1)\end{math}\\
\While{True}{
\For{all \begin{math}x \in X_f\end{math}}                    
{
\For{\begin{math}l \gets 0\end{math} to \begin{math} 1\end{math}}  
{
\begin{math} S(l,x) = g(x,\kappa^{-1} (\lfloor 2^{log M}l\rfloor) ) + \gamma S(\{2^{logM}l\},f(x, \kappa^{-1} (\lfloor 2^{log M}l\rfloor)))
\end{math}
}
}
}
\end{algorithm}
Once the exact Score-life function is computed, optimal Infinite Horizon Cost for a state $x$ can be computed by computing the minimum value of Score life function corresponding to the state $x$.
\begin{equation}
    J^*(x) = Min_{l\in[0,1]}(S(l,x))
\end{equation}
The optimal life value, $l^*(x)$ encodes the optimal infinite horizon action sequence starting from state $x$. The set of bits at the starting of $l_x^*$ corresponds to optimal action to be taken at state $x$. The optimal policy, $\pi^*(x)$ can hence be computed by extracting initial bits of the optimal life value $l$.
\begin{equation}
    \pi^*(x) = \kappa^{-1}(\lfloor 2^{logM} l^*(x) \rfloor)
\end{equation}
Unlike the value function where there is a single real number for every state, the Score-life function is a bounded unique function, for every state. Due to the recursive nature of Theorem 1, the Score-life function of a single state is a weighted superposition of Score-life function of all the other states. Hence a Score-life function of a single state contains information about all the other states. For instance, consider the Score-life function of the origin state of the cartpole dynamical system (Fig 1). The Score-life function of the origin state $x_0$ is combination of the Score-life function of two neighbouring states $x_1$ and $x_2$, where $x_1 = f(x_0,u_0)$ and $x_2 = f(x_0,u_1)$. The Score-life function of states $x_1$ and $x_2$, again is a combination of the Score-life function of the neighbouring states and so on. Due to this property, a fractal pattern emerges while computing Score-life function of the cartpole dynamical system. Another way of interpreting this result is that the Score-life functions of different states are related by amplitude and phase difference. We now introduce Theorem 2, which relates Score-life function of state \begin{math} x_0\end{math} and any state \begin{math} x_N\end{math}. See Appendix A.2 for the proof. 
\begin{theorem}
Let the dynamics of the environment satisfy the equation: \begin{math} x_{k+1} = f(x_k,u_k)\end{math}, where \begin{math}u_k \in U,x_k \in X, |U| = M  \end{math}. Let \begin{math}
    S(l,x_0)
\end{math} be the Score-life function of state \begin{math}
    x_0
\end{math}. Let \begin{math}
    \{u_0\} \{u_1\} \{u_2\}.... \{u_N\}
\end{math} be the action trajectory taken from state \begin{math}
    x_0
\end{math} and let \begin{math}
    \{x_0\} \{x_1\} \{x_2\}.... \{x_N\}
\end{math} be the corresponding state trajectory. Let \begin{math}
    \kappa 
\end{math} be a surjective mapping from action set to the set of binary sequences of length \begin{math}
    log(M)
\end{math}. Then the Score-life function of state \begin{math}
    x_N
\end{math} is given by: \\
\[S(l,x_N) =  \frac{1}{ \gamma^N} (S({\frac{l}{2^{N log M}} + \phi_N},x_0) - \psi_N ) \]
where
\[\phi_N = \sum_{i=0}^{N-1}2^{(N-i-1)log(M)}\kappa(u_i)  \]
and 
\[\psi_N = \sum_{i=0}^{N-1}\gamma^i g(x_i,u_i)  \]
\end{theorem}
Theorem 2 tells us that the Score-life functions of different states \begin{math}
    x_N
\end{math} and \begin{math}
    x_0
\end{math} are related by coordinated transformations. We can apply this result to design algorithms for efficient computation of optimal action sequences. Specifically, once we compute the Score-life function of a single state \begin{math}
    x_0
\end{math}, then for any different state \begin{math}
    x_N
\end{math}, we only have to compute the parameters \begin{math}
    N,\psi_N,\phi_N
\end{math} for estimating the Score-life function at state \begin{math}
    x_N
\end{math}. 

\section{Representation and computation
of the Score-life function}

The Score-life function need not necessarily be a polynomial or smooth function. In general, the Score-life function need not be continuous as well. Continuous Score-life functions are often fractal functions \cite{lagarias2011takagi,barnsley1986fractal}, which can be represented using Faber Schauder basis functions. In the next section, we discuss exact methods for computing coefficients of the Faber Schauder basis functions. The exact representation of Score-life function is expensive to compute and optimize, hence we also discuss approximate methods wherein a polynomial approximation of the Score-life function is used instead of the exact fractal function. \footnote{Code for this work is available at \url{https://github.com/Abhinav-Muraleedharan/Beyond_Dynamic_Programming.git}.}
\subsection{Exact Methods}
The Score-life function can be written in the Faber Schauder basis \cite{Faber1910,J1927} as:
\[S(l,x) = \alpha_0(x) + \alpha_1(x)l + \sum_{j=0}^{\infty} \sum_{i=0}^{2^j - 1} \alpha_{ij}(x) e_{ij}(l)\]
where \begin{math} e_{ij}(l) \end{math} are the Faber Schauder basis functions given by:
\[  e_{i,j}(l) = 2^j(|l - \frac{i}{2^j}| + |l -\frac{i+1}{2^j}| - |2l - \frac{2i + 1}{2^{j}}|) \] These functions form a Schauder basis of the Banach space \begin{math}
    C^0([0,1])
\end{math}, in the sup norm. For a detailed discussion of Faber Schauder bases, see \cite{megginson2012introduction}.
The coefficients of the faber schauder basis representation can be computed by the following set of equations.
\begin{equation}
    \alpha_0(x) = S(l=0,x)
\end{equation}
\begin{equation}
    \alpha_1(x) =S(l=1,x) - S(l=0,x)
\end{equation}
\begin{equation}
    \alpha_{ij}(x) = S(l=\frac{2i + 1}{2^{j+1}},x) - \frac{1}{2}(S(l=\frac{i}{2^j},x)+ S(l = \frac{i+1}{2^j},x)) 
\end{equation}

\begin{figure}[ht]
\begin{center}
\includegraphics[scale=0.5]{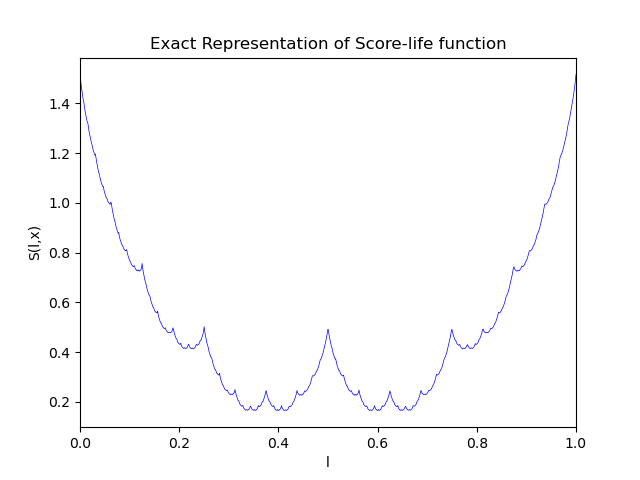}
\end{center}
\caption{ Exact representation of the Score life function of the origin state \begin{math}\textbf{x} = [x,\dot{x},\theta,\dot{\theta}] = [0,0,0,0]\end{math} of carpole dynamical system using computed Faber Schauder coefficients.}
\centering
\end{figure}
For a given state \begin{math}
    x
\end{math}, the coefficients of the Faber Schauder expansion of the function can be computed from eq(10-12). Evaluation of Score-life function at specific values involves computing an approximate value to the infinite horizon cost corresponding to the action sequence. Note that computation of Faber-Schauder coefficients involves exponential number of queries to the Score-life function, and hence it is computationally expensive. It is infeasible to compute Score-life function of every state using eq(10-12). So, we first compute the Faber-Schauder coefficients of state
\begin{math}
    x_0
\end{math} and then for any given state \begin{math}
    x_N
\end{math}, we compute the coefficients by applying Theorem 2. Specifically, the Faber Schauder representation of Score life function of state \begin{math}
    x_N
\end{math} can be written as:
\begin{equation}
 S(l,x_N) =  \frac{1}{ \gamma^N} (S({\frac{l}{2^{N log M}} + \phi_N},x_0) - \psi_N )
\end{equation}
After substituting the Faber Schauder expansion of Score-life function of state \begin{math} x_)\end{math}, eq(13) becomes: 
\begin{equation}
    S(l,x_N) = \frac{1}{\gamma^N} (\alpha_0(x_0) - \psi_N + \alpha_1(x_0)(\frac{l}{2^{N log(M)}}+\phi_N) + \sum_{j=0}^{\infty} \sum_{i=0}^{2^j - 1} \alpha_{ij}(x_0) e_{ij}(\frac{l}{2^{N log(M)}} +\phi_N))
\end{equation}
For any given state \begin{math}
    x_N
\end{math}, the values of \begin{math}
    \phi_N
\end{math},\begin{math}
    N
\end{math},\begin{math}
    \psi_N
\end{math} are unknowns. However, we can compute the values of these unknown variables by solving a set of nonlinear equations. Specifically, we can approximately evaluate the Score-life function at a finite number of values (minimum 3), and compute the values of 3 unknowns \begin{math}
    \phi_N
\end{math}, \begin{math}
    N
\end{math},\begin{math}
   \psi_N
\end{math}. After computing these unknown values, the optimal infinite action sequence from state \begin{math}
    x_N
\end{math} can be computed by computing the minimum of Score-life function of \begin{math}
    x_N
\end{math}.
For a detailed description of algorithm and pseudocode, refer Appendix B.1.
\begin{equation}
    l^*(x_N) = argmin_{l\in [0,1)} S(l,x_N)
\end{equation}
The optimal cost-to go is given by:
\begin{equation}
    J^*(x_N) = Min_{l\in [0,1)} S(l,x_N) = S(l^*(x_N),x_N)
\end{equation}
The interesting part of the Score-life programming method is that the optimal cost to go from a state can be computed directly from monte carlo rollouts, and we do not have to iterate through all states in the state space. Furthermore, the optimal infinite action sequence can be directly computed, without needing a policy function. Although this approach is theoretically valid, in practice, solving eq(14) and computing the minima of fractal Score-life function is computationally expensive. A more practical algorithm involves approximations to the Score-life function, which is discussed in the next section.
\subsection{Approximate Methods}
Instead of using exact fractal representation of the Score-life function, we can approximate the original Score-life function using a polynomial function that matches with the exact Score-life function at a finite number of points. Let the polynomial Score-life function be denoted by \begin{math}
    S_{poly}(l,x)
\end{math}. For a polynomial representation of the Score-life function, we require \begin{math}
    ||S_{poly}(l,x) - S(l,x)||_2 < \epsilon 
\end{math} for some \begin{math}
    \epsilon > 0
\end{math} and \begin{math}
    \forall l \in [0,1)
\end{math}.
Mathematically, the polynomial Score-life function can be written as:
\begin{equation}
    S_{poly}(l,x) = \sum_{i=0}^n a_i(x) l^i
\end{equation}
For a given state \begin{math}
    x
\end{math}, the coefficients of the polynomial function \begin{math}
    a_i(x)
\end{math} are unknowns. We evaluate the coefficients in a similar manner to exact methods, first we compute the coefficients of the polynomial of a particular state \begin{math}
    x_0
\end{math}. This can be performed by evaluating the Score-life function at a finite sample of \begin{math}
    l
\end{math} values and solving a non-linear regression problem to compute the coefficients. After computation of coefficients for a specific state \begin{math}
    x_0
\end{math}, we apply Theorem 2 to compute the coefficients of the Score-life function for different states in the state space.

\begin{figure}[ht]
\begin{center}
\includegraphics[scale=0.6]{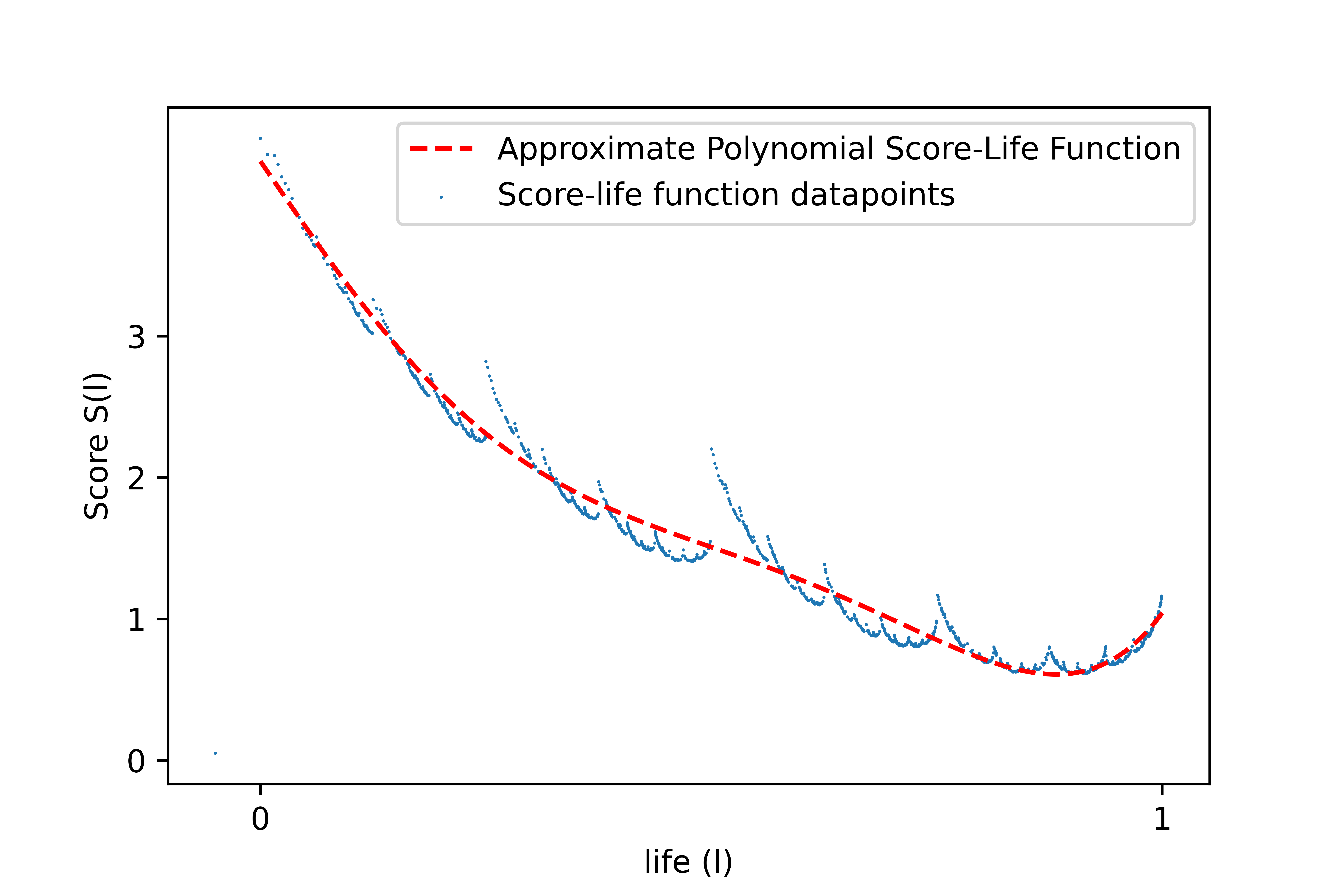}
\end{center}
\caption{ Approximate Polynomial function $S_{poly}(l,x)$ and the Score-Life Function $S(l,x)$ for the state $\textbf{x} =[x,\dot{x},\theta,\dot{\theta}]^T = [-0.0039, -0.3902,  0.0058,  0.5853]^T$ of the cartpole dynamical system. The actual Score-Life function is a fractal function, and the polynomial function is of degree 5. The minimum value of the polynomial approximation is close to the minimum value of the actual fractal Score-life function}
\centering
\end{figure}
The optimal infinite horizon cost from a given  state \begin{math} x \end{math} can then be computed by computing the minimum of the polynomial approximation of the Score-life function.
\begin{equation}
  J^*(x) \approx \min_{l\in[0,1)}S_{poly}(l,x)
\end{equation}
However, note that:
\begin{equation}
  l^*_{poly} \neq l^*
\end{equation}
 Although the minimum value of the polynomial function can be arbitrarily close to the minimum value of the actual Score-life function, the corresponding life value ($l$) need not be optimal. Hence, the polynomial approximation function cannot be used directly to compute the optimal infinite horizon action sequence from state $x$. The optimal action from any state $x$ can be computed by the Bellmann equation as follows:
\begin{equation}
\pi^*(x) = argmin_{a_i \in U}(g(x,a_i) + \min_{l\in[0,1]}\gamma S_{poly}(l,f(x,a_i))
\end{equation}
For a detailed description of algorithm and pseudocode, refer Appendix B.2.
\section{Simulation Results}
\noindent We applied approximate and exact methodstp solve the cart pole balancing task in OpenAI Gym platform \cite{brockman2016openai}.In our approach, we focused on a model-based deterministic scenario, where the agent had complete knowledge of the system dynamics and state. For exact methods, we optimized fractal functions to compute finite horizon action sequences. When the Score-life function is approximated with a polynomial function, we used the approximate Score-life function to compute the optimal cost-to-go of neighboring states and computed optimal action to minimize cumulative cost. We noticed that efficiently optimizing fractal functions is a hard computational problem, and most of the time the solver gets trapped in local minima. Hence we observe poor performance for exact methods in comparison to approximate methods. See Appendix C for detailed discussion and trajectory plots.  
\begin{figure}[htbp]
  \centering
  \begin{subfigure}[b]{0.45\textwidth}
    \includegraphics[width=\textwidth]{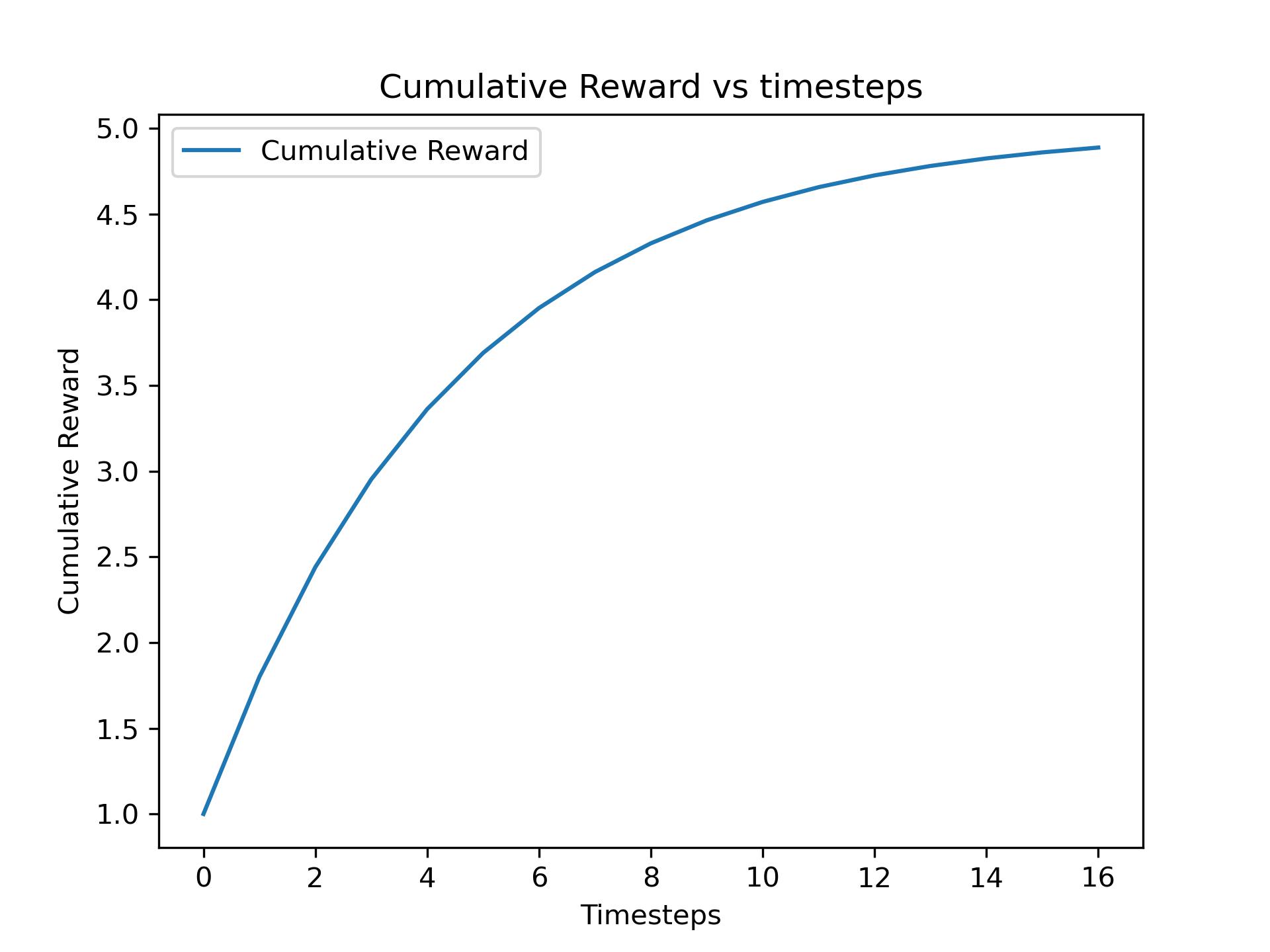}
    \caption{Cumulative Reward vs timesteps(Exact Method)}
    \label{fig:image1}
  \end{subfigure}
  \hfill
  \begin{subfigure}[b]{0.45\textwidth}
    \includegraphics[width=\textwidth]{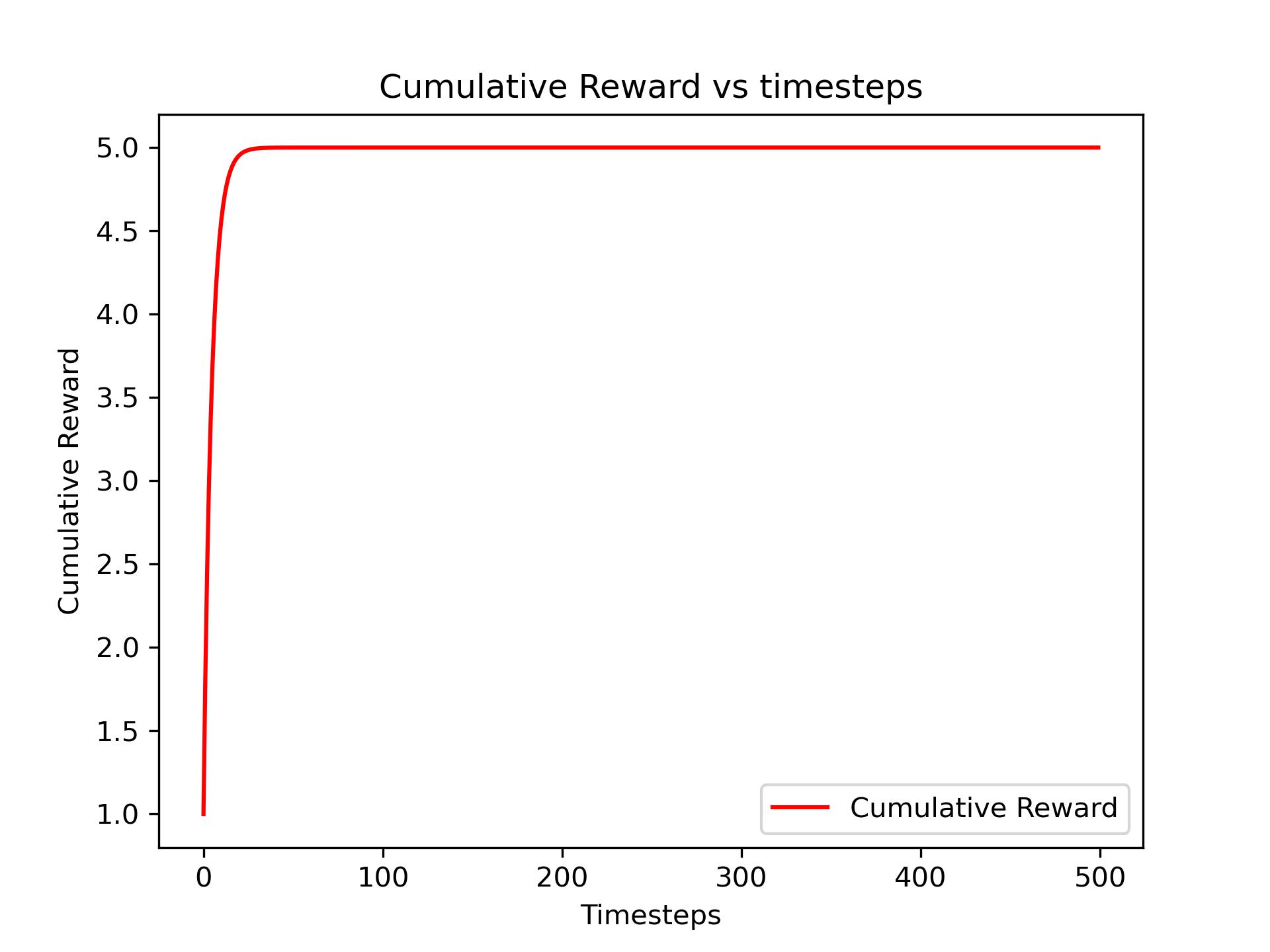}
    \caption{Cumulative Reward vs timesteps (Approximate Method)}
    \label{fig:image2}
  \end{subfigure}
  \caption{Cumulative Reward vs timesteps for cart pole environment. Figure (a) shows the results of the exact method and Figure (b) shows the results of the approximate method}
  \label{fig:side_by_side}
\end{figure}

\subsection{Simulation Setup}
\noindent The cartpole system available in openai gym environment was used for simulation. The agent has access to the position and velocity of the cart as state and can only go left or right for each action. The agent receives a reward of +1 for every timesteps and the episode terminates when the cart goes out of the boundaries or when the pole falls.
\subsection{Performance of Algorithms}
\noindent The approximate method was able to stabilize the cartpole for 500 timesteps while the exact method was only able to stabilize the cartpole for 16 timesteps. For a wide range of initial conditions, approximate methods outperformed exact methods. The computation of optimal action sequences via exact methods is computationally expensive. We used a gradient-based method for computing the minima of fractal functions. (See Algorithm 3). Computation of gradient of the fractal function require \begin{math}
    O(2^n)
\end{math} iterations, where \begin{math}
    n
\end{math} is the order of fractal function. On the other hand, in the case of approximate methods, for low-degree polynomials, the minima can be computed in closed form. We found that quadratic functions are sufficient for cart pole balancing. Since the minima of quadratic functions can be computed efficiently, approximate methods show superior performance in this task. 

\section{Conclusion and Future Work}
We presented Score-life programming, a novel theoretical approach in solving infinite horizon decision making problems in reinforcement learning. Our method can effectively search over non stationary policies and can compute infinite horizon action sequences directly from a given state input. This paper also laid the groundwork for a new class of no-policy reinforcement learning algorithms, and also showed the connection between number theoretic fractal functions and reinforcement learning problems. To conclude, our work contributes significant insights into the structure of dynamic programming methods and suggests novel methods to solve Dynamic Programming problems. In future work, we plan to apply the method to model free and stochastic reinforcement learning problems. 
\\

\section*{Acknowledgements}
I would like to thank Prof. Andrew Goldenberg, Prof. Florian Shkurti, and Prof. Prasanth Nair for helpful discussions.
\bibliographystyle{unsrt}  
\bibliography{references}  

\begin{thebibliography}{10}

\bibitem{sutton2018reinforcement}
Richard~S Sutton and Andrew~G Barto.
\newblock {\em Reinforcement learning: An introduction}.
\newblock MIT press, 2018.

\bibitem{silver2016mastering}
David Silver, Aja Huang, Chris~J Maddison, Arthur Guez, Laurent Sifre, George
  Van Den~Driessche, Julian Schrittwieser, Ioannis Antonoglou, Veda
  Panneershelvam, Marc Lanctot, et~al.
\newblock Mastering the game of go with deep neural networks and tree search.
\newblock {\em nature}, 529(7587):484--489, 2016.

\bibitem{andrychowicz2020learning}
OpenAI:~Marcin Andrychowicz, Bowen Baker, Maciek Chociej, Rafal Jozefowicz, Bob
  McGrew, Jakub Pachocki, Arthur Petron, Matthias Plappert, Glenn Powell, Alex
  Ray, et~al.
\newblock Learning dexterous in-hand manipulation.
\newblock {\em The International Journal of Robotics Research}, 39(1):3--20,
  2020.

\bibitem{li2017deep}
Yuxi Li.
\newblock Deep reinforcement learning: An overview.
\newblock {\em arXiv preprint arXiv:1701.07274}, 2017.

\bibitem{bellman1966dynamic}
Richard Bellman.
\newblock Dynamic programming.
\newblock {\em Science}, 153(3731):34--37, 1966.

\bibitem{sutton1988learning}
Richard~S Sutton.
\newblock Learning to predict by the methods of temporal differences.
\newblock {\em Machine learning}, 3:9--44, 1988.

\bibitem{gordon1995stable}
Geoffrey~J Gordon.
\newblock Stable function approximation in dynamic programming.
\newblock In {\em Machine learning proceedings 1995}, pages 261--268. Elsevier,
  1995.

\bibitem{ernst2005tree}
Damien Ernst, Pierre Geurts, and Louis Wehenkel.
\newblock Tree-based batch mode reinforcement learning.
\newblock {\em Journal of Machine Learning Research}, 6, 2005.

\bibitem{munos2008finite}
R{\'e}mi Munos and Csaba Szepesv{\'a}ri.
\newblock Finite-time bounds for fitted value iteration.
\newblock {\em Journal of Machine Learning Research}, 9(5), 2008.

\bibitem{mnih2015human}
Volodymyr Mnih, Koray Kavukcuoglu, David Silver, Andrei~A Rusu, Joel Veness,
  Marc~G Bellemare, Alex Graves, Martin Riedmiller, Andreas~K Fidjeland, Georg
  Ostrovski, et~al.
\newblock Human-level control through deep reinforcement learning.
\newblock {\em nature}, 518(7540):529--533, 2015.

\bibitem{bertsekas2022abstract}
Dimitri Bertsekas.
\newblock {\em Abstract dynamic programming}.
\newblock Athena Scientific, 2022.

\bibitem{lagarias2011takagi}
Jeffrey~C Lagarias.
\newblock The takagi function and its properties.
\newblock {\em arXiv preprint arXiv:1112.4205}, 2011.

\bibitem{barnsley1986fractal}
Michael~F Barnsley.
\newblock Fractal functions and interpolation.
\newblock {\em Constructive approximation}, 2:303--329, 1986.

\bibitem{Faber1910}
Georg Faber.
\newblock Über die orthogonalfunktionen des herrn haar.
\newblock {\em Jahresbericht der Deutschen Mathematiker-Vereinigung},
  19:104--112, 1910.

\bibitem{J1927}
Schauder J.
\newblock Zur theorie stetiger abbildungen in funktionalräumen.
\newblock {\em Mathematische Zeitschrift}, 26:47--65, 1927.

\bibitem{megginson2012introduction}
Robert~E Megginson.
\newblock {\em An introduction to Banach space theory}, volume 183.
\newblock Springer Science \& Business Media, 2012.

\bibitem{brockman2016openai}
Greg Brockman, Vicki Cheung, Ludwig Pettersson, Jonas Schneider, John Schulman,
  Jie Tang, and Wojciech Zaremba.
\newblock Openai gym.
\newblock {\em arXiv preprint arXiv:1606.01540}, 2016.

\end{thebibliography}
\newpage
\appendix 
\section{Definitions and Results}
Any real number \begin{math}
    l
\end{math} in the interval \begin{math}
    [0,1)
\end{math} can be expressed in binary as:\[l = \sum_{i=0}^{\infty} 2^{-i}d_i\] where \begin{math}
    d_i \in \{0,1\}
\end{math}. The sequence of binary digits \begin{math}
    \{d_i\}_{i=0}^{\infty}
\end{math} can be mapped onto an action sequence taken by an agent, and hence, any action sequence can be mapped to a real number within the interval \begin{math}
    [0,1)
\end{math}. We construct this mapping by defining a function \begin{math}
    \kappa 
\end{math} which maps action values to binary sequences of length \begin{math}
    log M
\end{math}, where \begin{math}
    M 
\end{math} is the cardinality of the action set \begin{math}
    U
\end{math}. For instance, consider an agent with \begin{math}
    U = \{-10,-5,5,10\}
\end{math}. In this case, we can define a \begin{math}
    \kappa: U \rightarrow \mathbb{N}
\end{math}, \begin{math}
    \kappa (-10) = 00 = 0
\end{math},
\begin{math}
\kappa (-5) = 01 = 1,
\end{math}
\begin{math}
    \kappa(5) = 10 = 2
\end{math},
\begin{math}
    \kappa(10) = 11 = 3
\end{math}. The definition of \begin{math}
    \kappa 
\end{math} is arbitrary, and for an action set of cardinality \begin{math}
    M
\end{math}, there exists \begin{math}
    M!
\end{math} possible definitions for \begin{math}
    \kappa
\end{math}. Also, \begin{math}
    \kappa (u_i) \in \{0,1,2,...M-1\}
\end{math}. Now we define a real number \begin{math}
    l \in [0,1)
\end{math}, which is given by:
\[l =\sum_{i=0}^{\infty} 2^{logM(-i-1)} \kappa(u_i)\]

\begin{proposition}
\label{sec: Proposition A.1}
Let \begin{math}
    \{u_0\}\{u_1\}\{u_2\}...\{u_{\infty}\} 
\end{math} be the infinite horizon action sequence taken by an agent \begin{math}
    A
\end{math}, and let \begin{math}
    u_k \in U 
\end{math} where \begin{math}
    U
\end{math} is the set of all possible actions, and \begin{math}
    |U| = M
\end{math}. Then the life-value \begin{math}
    l =\sum_{i=0}^{\infty} 2^{logM(-i-1)}{\kappa(u_i)}
\end{math} lies in the interval [0,1).
 \end{proposition}
\begin{proof}
The minimum value of \begin{math}
    l
\end{math} is attained when \begin{math}
    \kappa(u_i) = 0
\end{math}, \begin{math}
    \forall i 
\end{math}. In this case,

\[
l = \kappa(u_i)\sum_{i=0}^{\infty} 2^{logM(-i-1)}
\]
Since \begin{math}
    \kappa(u_i) = 0
\end{math}, we get \begin{math}
    l_{min} = 0
\end{math}. 
The maximum value is attained when \begin{math}
    \kappa(u_i) 
\end{math} is maximized, and the same for all \begin{math}
    i
\end{math}. This happens when \begin{math}
    \kappa(u_i) = M-1
\end{math}. In this case, the life-value is given by:
\[
l_{max} = (M-1)\sum_{i=0}^{\infty} 2^{logM(-i-1)}
\]
The infinite sum is a geometric series and the limit of the sum converges to:
\[
\lim_{{N \to \infty}}\sum_{i=0}^{N} 2^{logM(-i-1)} \rightarrow \frac{2^{-log M}}{1 - 2^{-log M}}
\]
Since we are taking \begin{math}
    log
\end{math} w.r.t base 2, the expression can be simplified to:
\[
  \lim_{{N \to \infty}}  \sum_{i=0}^{N} 2^{logM(-i-1)} \rightarrow \frac{1}{M-1}
\]
Substituting the result in the expression for \begin{math}
    l_{max}
\end{math}, we get
\[
l_{max} \rightarrow 1
\] Since the upper limit of life value is \begin{math}
    1
\end{math}, and since the minimum value is \begin{math}
    0
\end{math}, the life-value \begin{math}
    l = \sum_{i=0}^{\infty} 2^{logM(-i-1)}{\kappa(u_i)}
\end{math} lies in the interval \begin{math}
    [0,1)
\end{math}.
\end{proof}
 
\subsection{Score-life function Iteration}
\label{sec: A.1}
In this section, we prove Theorem \ref{thm:1}, which relates the Score-life functions of neighboring states, dynamics equation, and stage cost function. 
\setcounter{theorem}{0}
\begin{theorem}\label{thm:1}
Let the dynamics of the environment satisfy the equation: \begin{math} x_{k+1} = f(x_k,u_k)\end{math}, where \begin{math}u_k \in U,x_k \in X, |U| = M  \end{math}. Let the infinite horizon cost be given by \begin{math} J_{\infty}(x_0) =  \sum_{k=0}^{\infty} \gamma^k.{g(x_k,u_k)} \end{math}, and let \begin{math} \kappa\end{math} be the mapping from set \begin{math}U\end{math} to binary digit sequences. Then the Score-life function of the dynamical system obeys the following recursive equation:\\
\[S(l,x) = g(x,\kappa^{-1} (\lfloor 2^{log M}l\rfloor) ) + \gamma S(\{2^{logM}l\},f(x, \kappa^{-1} (\lfloor 2^{log M}l\rfloor))) \]

\end{theorem}
\begin{proof}
  For an infinite horizon action sequence \begin{math}
    \{u_0\}\{u_1\}\{u_2\}...\{u_{\infty}\} 
\end{math} taken by an agent, the life value is given by
\begin{math}
    l =\sum_{i=0}^{\infty} 2^{logM(-i-1)}{\kappa(u_i)}
\end{math}. The Score-life function for the state \begin{math}
    x_0
\end{math}can be written as:
\[
S(l,x_0) = \sum_{i=0}^{\infty} \gamma^i g(x_i,u_i)
\]
Now, we can write rearrange the terms in the summation as:
\begin{equation}
S(l,x_0) = g(x_0,u_0) + \gamma\sum_{i=1}^{\infty} \gamma^{i-1} g(x_i,u_i)
\end{equation}
But the infinite sum
\begin{math}
    \sum_{i=1}^{\infty} \gamma^{i-1} g(x_i,u_i)
\end{math} is equal to the Score-life function of state \begin{math}
    x_1
\end{math} evaluated at a different life-value \begin{math}
    l'
\end{math}. Specifically, 
\begin{equation}
 \sum_{i=1}^{\infty} \gamma^{i-1} g(x_i,u_i) = S(l',x_1)
\end{equation}
where \begin{math}
    x_1 = f(x_0,u_0)
\end{math} and \begin{math}
    l' = \sum_{i=1}^{\infty} 2^{logM(-i)}{\kappa(u_i)}
\end{math}. Now we can derive a relation between \begin{math}
    l
\end{math} and \begin{math}
    l'
\end{math}. Recall,
\[
    l =   \sum_{i=0}^{\infty} 2^{log M(-i-1)} \kappa(u_{i})
\]
We can rearrange this equation by writing the first term separately and taking out \begin{math}
    2^{-logM}
\end{math} from the infinite summation. 
\[
    l = 2^{-log M} \kappa(u_0) + 2^{-log M} \sum_{i=0}^{\infty} 2^{log M(-i-1)} \kappa(u_{i+1})
\]
Note that the second infinite term summation term is nothing but \begin{math}
    l'
\end{math}. Hence, we can write \begin{math}
    l
\end{math} as:
\[
l = 2^{-log M} \kappa(u_0) + 2^{-log M} l'
\]
After rearranging, 
\[
l' = 2^{logM}l - \kappa(u_0)
\]
Now, note that \[
2^{logM}l = \kappa(u_0) + \sum_{i=0}^{\infty} 2^{log M(-i-1)}\kappa(u_{i+1})
\]Since \begin{math}
    \sum_{i=0}^{\infty} 2^{log M(-i-1)}\kappa(u_{i+1}) < 1
\end{math} due to Proposition \ref{sec: Proposition A.1}, we have:
\[
\lfloor 2^{logM}l \rfloor = \kappa (u_0) \implies u_0 = \kappa^{-1} (\lfloor 2^{logM}l \rfloor)
\] 
and 
\[
\{ 2^{logM}l \} = \sum_{i=0}^{\infty} 2^{log M(-i-1)}\kappa(u_{i+1}) = l'
\]
Substituting these results in eq(22) and eq(21) gives:
\[
S(l,x_0) = g(x_0,\kappa^{-1} (\lfloor 2^{log M}l\rfloor) ) + \gamma S(\{2^{logM}l\},f(x_0, \kappa^{-1} (\lfloor 2^{log M}l\rfloor))) 
\]
Since the result is true for any \begin{math}
    x_0
\end{math} in the state space, this proves our claim.
\end{proof}

\subsection{Relation between Score-life functions of any two states}
In the previous section, we derived the relation between the Score-life functions of neighboring states. In this section, we extend the previous result to derive a relationship between the Score-life functions of any two states in the state space. 
\begin{theorem}\label{thm:2}
Let the dynamics of the environment satisfy the equation: \begin{math} x_{k+1} = f(x_k,u_k)\end{math}, where \begin{math}u_k \in U,x_k \in X, |U| = M  \end{math}. Let \begin{math}
    S(l,x_0)
\end{math} be the Score-life function of state \begin{math}
    x_0
\end{math}. Let \begin{math}
    \{u_0\} \{u_1\} \{u_2\}.... \{u_N\}
\end{math} be the action trajectory taken from state \begin{math}
    x_0
\end{math} and let \begin{math}
    \{x_0\} \{x_1\} \{x_2\}.... \{x_N\}
\end{math} be the corresponding state trajectory. Let \begin{math}
    \kappa 
\end{math} be a surjective mapping from action set to the set of binary sequences of length \begin{math}
    log(M)
\end{math}. Then the Score-life function of state \begin{math}
    x_N
\end{math} is given by: \\
\[S(l,x_N) =  \frac{1}{ \gamma^N} (S({\frac{l}{2^{N log M}} + \phi_N},x_0) - \psi_N ) \]
where
\[\phi_N = \sum_{i=0}^{N-1}2^{(N-i-1)log(M)}\kappa(u_i)  \]
and 
\[\psi_N = \sum_{i=0}^{N-1}\gamma^i g(x_i,u_i)  \]
\end{theorem}

\begin{proof}
 First, we group the first \begin{math}
    N
\end{math} stage cost terms, and write the Score-life function as:
\begin{equation}
S(l,x_0) = \sum_{i=0}^{N-1} \gamma^i g(x_i,u_i) + \gamma^N (\sum_{i=0}^{\infty} \gamma^i g(x_{i+N},u_{i+N}) )
\end{equation}
Note that the infinite summation term is equal to the Score-life function of the state \begin{math}
    x_N
\end{math} evaluated at \begin{math}
    l_N = \sum_{i=0} ^{\infty} 2^{ (-i-1) log M}\kappa(u_{i+N})
\end{math}. Substituting this result in eq(23) gives:
\begin{equation}
    S(l,x_0) = \sum_{i=0}^{N-1} \gamma^i g(x_i,u_i) + \gamma^N S(l_N,x_N)
\end{equation}
Now we derive a relation between \begin{math}
    l_N
\end{math} and \begin{math}
    l
\end{math}. By grouping first \begin{math}
    N
\end{math} terms, we can write the equation of l as:
\[
l = \sum_{i=0}^{N-1} 2^{(-i-1)log{M}} \kappa(u_i) +  2^{-N log M} \sum_{i=0}^{\infty} 2^{(-i-1)log M}\kappa(u_{i+N})
\]
\[
l = \sum_{i=0}^{N-1} 2^{(-i-1)log{M}} \kappa(u_i) +  2^{-Nlog M} l_N
\]
After further algebraic rearrangements and substituting \begin{math}
    \phi_N = \sum_{i=0}^{N-1} 2^{(-i-1)log{M}} \kappa(u_i)
\end{math}, we get:
\[
l_N = 2^{N log M} (l - \phi_N)
\]Now we make the substitution \begin{math}
    \psi_N = \sum_{i=0}^{N-1}\gamma^i g(x_i,u_i)   
\end{math} in eq(24), and we get:
\begin{equation}
    S(l,x_0) = \psi_N + \gamma^N S(2^{N log M} (l - \phi_N),x_N)
\end{equation}
After change of variables, we get:
\begin{equation}
 S(l',x_N) =  \frac{1}{ \gamma^N} (S({\frac{l'}{2^{N log M}} + \phi_N},x_0) - \psi_N ) 
\end{equation}
Or,
\begin{equation}
 S(l,x_N) =  \frac{1}{ \gamma^N} (S({\frac{l}{2^{N log M}} + \phi_N},x_0) - \psi_N ) 
\end{equation}
\end{proof}

\section{Algorithms}
In this section, we present further details of Algorithms presented for the computation of the Score-life function and optimal infinite horizon action sequence. 
\subsection{Exact Methods}
In exact methods, we represent the Score-life function using the Faber-Schauder basis and compute the coefficients corresponding to Faber-Schauder basis functions. Specifically, we write:
\[
S(l,x) \approx \alpha_0(x) + \alpha_1(x)l + \sum_{j=0}^{n} \sum_{i=0}^{2^j - 1} \alpha_{ij}(x) e_{ij}(l)
\]
where \begin{math}
    e_{ij}(l)
\end{math} are Faber Schauder basis functions, and the coefficients \begin{math}
    \alpha_0, \alpha_1 
\end{math} and \begin{math}
    \alpha_{ij} 
\end{math} can be computed using eq(10-12). 
Computation of Faber Schauder coefficients is expensive and involves \begin{math}
    \mathcal{O}(2^n)
\end{math} queries to the Score-life function. This is because \begin{math}
    n^{th}
\end{math}   order approximation of the Score-life function involves \begin{math}
    \mathcal{O}(2^n)
\end{math} summation terms, with different coefficients for each term. 
\begin{algorithm}[H]
\caption{Computation of Faber Schauder Coefficients for any state \begin{math}
    x_0
\end{math}}
\SetAlgoLined
\textbf{Input: }
 Dynamics Model \begin{math}x_{k+1} = f(x_k,u_k)\end{math}, State \begin{math}x\subseteq X \end{math}, Approximate Score-life function estimator \begin{math}
     S(l,x) = \sum_{i=0}^n \gamma^i g(x_i,u_i)
 \end{math}\\
   \textbf{Result: }
Coefficients of Faber Schauder representation of the Score-life function \\
 
\textbf{Initialize: } \begin{math}i,j = 0\end{math}, \begin{math}\alpha_0,\alpha_1,\alpha_{i}{j}= 0\end{math},\begin{math}\forall x \in X_f\end{math}, \\
\begin{math}
     \alpha_0(x_0) = S(l=0,x)   
\end{math}
\\
\begin{math}
    \alpha_1(x_0) =S(l=1,x) - S(l=0,x)
\end{math}\\
\While{j < n}
{
\While{i < \begin{math}
    2^j -1
\end{math}} {
{\begin{math}
     \alpha_{ij}(x_0) = S(l=\frac{2i + 1}{2^{j+1}},x_0) - \frac{1}{2}(S(l=\frac{i}{2^j},x_0)+ S(l = \frac{i+1}{2^j},x_0))    
\end{math}}

}

}
\end{algorithm}
Hence computing Faber Schauder coefficients for all states and storing them is infeasible. Instead, we compute Faber Schauder coefficients for a single state, and for other states, we apply Theorem 2, and compute \begin{math}
    \phi_N
\end{math}, \begin{math}
    \psi_N
\end{math} and \begin{math}
    N
\end{math}. Note that for any valid solution, \begin{math}
    N > 0
\end{math}. To compute \begin{math}
    \phi_N
\end{math}, \begin{math}
    \psi_N
\end{math} and \begin{math}
    N
\end{math}, for a given state \begin{math}
    x_N
\end{math}, we have to solve the nonlinear equation:
\begin{equation}
 S(l,x_N) =  \frac{1}{ \gamma^N} (S({\frac{l}{2^{N log M}} + \phi_N},x_0) - \psi_N ) 
\end{equation}
For general Score-life functions, closed form solution doesn't exist, and we would have to compute the unknowns by using nonlinear regression methods. Specifically, we evaluate 
the Score-life function \begin{math}
    S(l,x_N)
\end{math} at datapoints \begin{math}
    \{l_i\}_{i=0}^n
\end{math}, and construct the dataset \begin{math}
    D = \{l_i, S(l_i,x_N)\}_{i=0}^{n}
\end{math} for a finite number of samples \begin{math}
    n
\end{math}. Then, we minimize the following objective function to compute \begin{math}
    \theta = [\phi_N,\psi_N,N]^T
\end{math}
\begin{equation*}
\min_{\theta} \sum_{i=1}^{n} \left(S(l_i,x_N;\theta) - S(l_i,x_N)\right)^2
\end{equation*}
After computing the unknown parameters \begin{math}
    \phi_N
\end{math}, \begin{math}
    \psi_N
\end{math} and \begin{math}
    N
\end{math}, to estimate the optimal infinite horizon action sequence, we compute the minima of the Score-life function.
\begin{equation}
    l^*(x_N) = argmin_{l\in [0,1)} S(l,x_N)
\end{equation}
Computing the optimal \begin{math}
     l^*(x_N)
\end{math} and computing the unknown parameters 
After computing the unknown parameters \begin{math}
    \phi_N
\end{math}, \begin{math}
    \psi_N
\end{math} and \begin{math}
    N
\end{math}, would require optimizing and differentiating fractal functions, which we discuss in the next section.
\subsubsection{Optimizing Fractal Functions}
The Score-life functions are typically fractal functions, and fractal functions are hard to optimize. Technically, fractal functions are non-differentiable everywhere, and hence it is not theoretically valid to apply gradient based techniques to fractal functions. However, in practice, for an \begin{math}
    n^{th}
\end{math} order approximation of a fractal function, we can evaluate the derivative at \begin{math}
    l
\end{math} values, by estimating:
\[
\frac{\partial S(l, x)}{\partial l} \approx   \alpha_1(x) + \sum_{j=0}^{n} \sum_{i=0}^{2^j - 1} \alpha_{ij}(x) e_{ij}'(l)
\]
where \begin{math}
     e_{ij}'(l)
\end{math}
is given by:
\[
 e_{ij}'(l) = \frac{\partial}{\partial l} (2^j(|l - \frac{i}{2^j}| + |l -\frac{i+1}{2^j}| - |2l - \frac{2i + 1}{2^{j}}|))
\]
Note that theoretically, the derivative of the function is \begin{math}
    |al -b| 
\end{math} is not defined at \begin{math}
    l = \frac{b}{a}
\end{math}. In practice, at \begin{math}
    l = \frac{b}{a}
\end{math}, we set the derivative \begin{math}
    \frac{\partial}{\partial l} (|al -b|) = a
\end{math} when \begin{math}
     l = \frac{b}{a}
\end{math}. This is similar to estimating the derivative of Relu activation functions in deep neural networks. Empirically, we have found that this approach can succesfully find the minima, but the precision of \begin{math}
    l^*(x_N)
\end{math} is relatively low. 

\begin{algorithm}[H]
\caption{Computation of optimal $l^*(x)$ given Faber Schauder Coefficients for any state $x$}
\SetAlgoLined
\textbf{Input:} Faber Schauder representation of the Score-life function $S(l,x)$, for state $x\subseteq X$ \\
\textbf{Result:} Optimal $l^*$ \\
\textbf{Initialize:} $\eta = 0.001$, $\delta = 0.01$, $i=0$, $l \sim \text{Uniform}(0, 1)$ \\
\While{$\left(\frac{\partial S(l, x)}{\partial l}\right)^2 \geq \delta$}{
    $g_i = \frac{\partial S(l, x)}{\partial l}$ \\
    $l = l - \eta g_i$ \\
    \If{$g_{i-1}\cdot g_i < 0$}{\textbf{break}}
    $i \leftarrow i + 1$
}
\end{algorithm}

\subsection{Approximate Methods}
In approximate methods, we approximate the Score-life function using a polynomial of degree \begin{math}
    N_{poly}
\end{math}, and use it to compute the cost to go from a given state \begin{math}
    x
\end{math}. Then, to compute optimal actions, we estimate:
\begin{equation}
\pi^*(x) = argmin_{a_i \in U}(g(x,a_i) + \min_{l\in[0,1]}\gamma S_{poly}(l,f(x,a_i))
\end{equation}
The minimum of polynomial function \begin{math}
    S_{poly}
\end{math} can be evaluated in closed form, and hence in practice, approximate methods are more efficient. However, note that we would not be able to compute action sequences using approximate methods. 
\begin{algorithm}[H]
\caption{Computation of polynomial approximation of the Score life function \begin{math}
    x
\end{math}}
\SetAlgoLined
\textbf{Input: }
 Dynamics Model \begin{math}x_{k+1} = f(x_k,u_k)\end{math}, State \begin{math}x\subseteq X \end{math}, Approximate Score-life function estimator \begin{math}
     S(l,x) = \sum_{i=0}^n \gamma^i g(x_i,u_i)
 \end{math}\\
   \textbf{Result: }
 Polynomial approximation of Score-life value function of state \begin{math}x\end{math}\\
 
\textbf{Initialize: } \begin{math}\alpha_i(x)= 0\end{math}, \\
\While{i < n}
{
\begin{math}
    l_i \sim \text{Uniform}(0, 1)
\end{math}\\
\Comment{//Evaluate Score-life function at \begin{math}
    l_i
\end{math}}\\
\begin{math}
    y_i = S(l_i,x)
\end{math}\\
\Comment{//Store \begin{math}
   \{ l_i,y_i\}
\end{math} in a dataset D.} 
}
\Comment{Compute coefficients 
\begin{math}
    \alpha_i
\end{math}
of polynomial representation of the Score-life function 
\begin{math}
    S_{poly}
\end{math}
}
\\
\begin{math}
    \alpha = argmin_{\alpha} \sum_{i=0}^{n-1} \left(S_{poly}(l_i,x;\alpha) - S(l_i,x)\right)^2
\end{math}
\\
\end{algorithm}
After computing polynomial approximation for a state \begin{math}
    x_0
\end{math}, polynomial Score-life function of a given state \begin{math}
    x_N
\end{math} can be computed by applying Theorem 2 and computing \begin{math}
    \phi_N
\end{math}, \begin{math}
    \psi_N
\end{math}, \begin{math}
    N
\end{math}.
\begin{equation}
 S_{poly}(l,x_N) =  \frac{1}{ \gamma^N} (S_{poly}({\frac{l}{2^{N log M}} + \phi_N},x_0) - \psi_N ) 
\end{equation}
To solve eq(31), we evaluate 
the Score-life function \begin{math}
    S(l,x_N)
\end{math} at datapoints \begin{math}
    \{l_i\}_{i=0}^n
\end{math}, and construct the dataset \begin{math}
    D = \{l_i, S(l_i,x_N)\}_{i=0}^{n}
\end{math} for a finite number of samples \begin{math}
    n
\end{math}. Then, we minimize the following objective function to compute \begin{math}
    \theta = [\phi_N,\psi_N,N]^T
\end{math}
\begin{equation*}
\min_{\theta} \sum_{i=1}^{n} \left(S_{poly}(l_i,x_N;\theta) - S(l_i,x_N)\right)^2
\end{equation*}
After computing the parameters, \begin{math}
    \theta = [\phi_N,\psi_N,N]^T
\end{math}, we can compute the optimal actions by using eq (30). As the degree of the polynomial is increased, the optimal value of \begin{math}
    S_{poly}(l,x)
\end{math} approaches optimal infinite horizon cost \begin{math}
    J^*(x)
\end{math}. 
\section{Simulation Results}

In this section, we report the simulation results of running our algorithms on the cart pole dynamical system. We conducted several experiments for different values of hyperparameters and stage cost functions. We noticed that in practice, approximate methods outperform exact methods, and can efficiently compute optimal policies . 
\subsection{Monte-Carlo Experiments}

We conducted monte carlo simulations to estimate Score-life function of the cart pole dynamical system for various values of discount factor \begin{math}
    \gamma 
\end{math}. We used \begin{math}
    x^T Q x
\end{math} as the stage cost function. Where \begin{math}
    {x} = [x,\dot{x},\theta,\dot{\theta}]^T
\end{math} is the 4 dimensional state of the cart pole system and \begin{math}
    Q_{diag}  = [2,1,8,1]
\end{math}. For small \begin{math}
    \gamma 
\end{math} values, the fractal function is well-behaved, however for larger \begin{math}
    \gamma 
\end{math} values, the oscillatory behavior of the fractal function increases and hence it becomes harder to compute global optima (Fig 5).
\begin{figure}
    \centering   
    \begin{subfigure}{0.45\textwidth}
        \includegraphics[width=\linewidth]{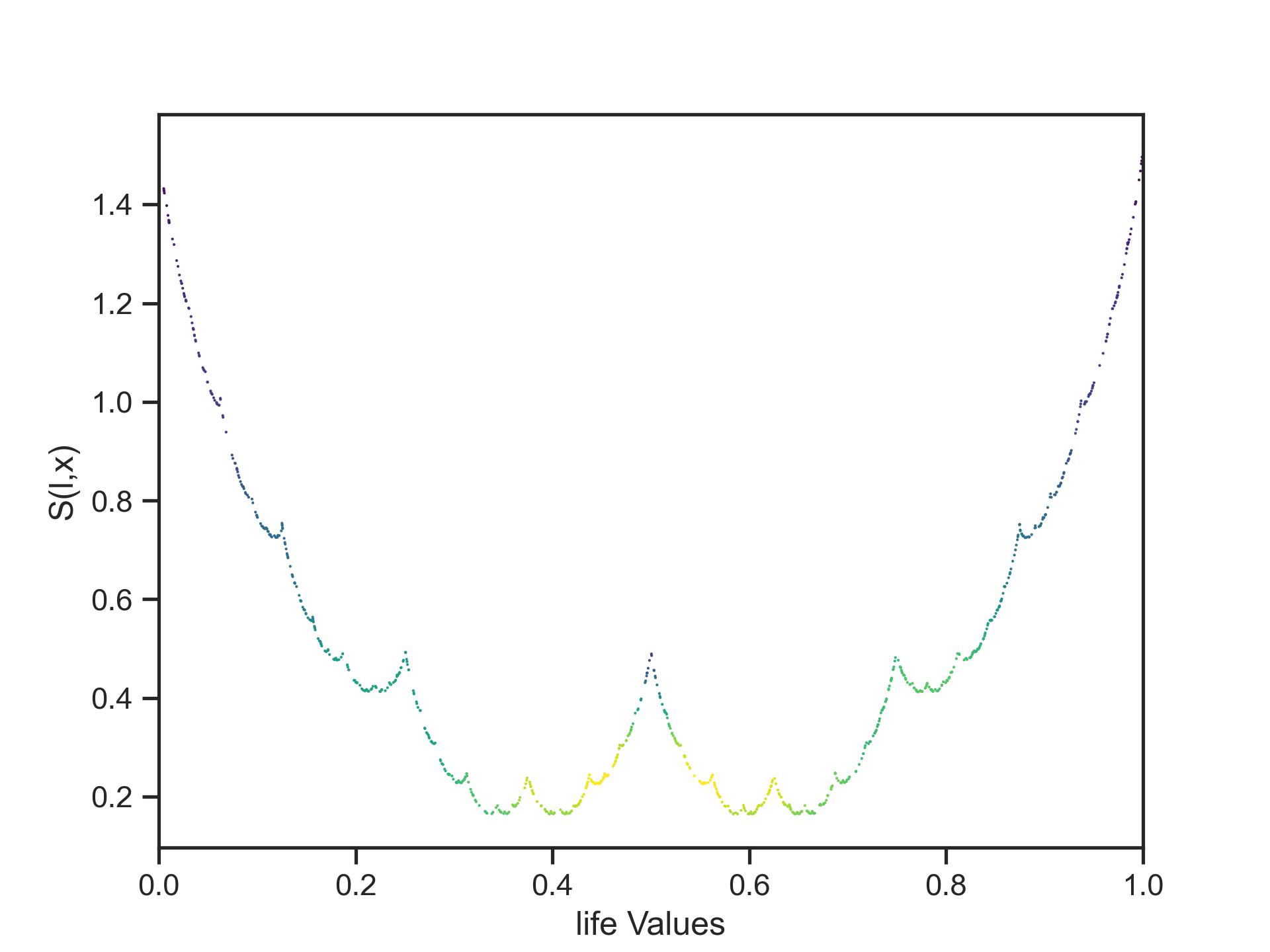}
        \caption{\begin{math}
            \gamma = 0.5
        \end{math}}
        \label{fig:plot1}
    \end{subfigure}
    \hfill
    \begin{subfigure}{0.45\textwidth}
        \includegraphics[width=\linewidth]{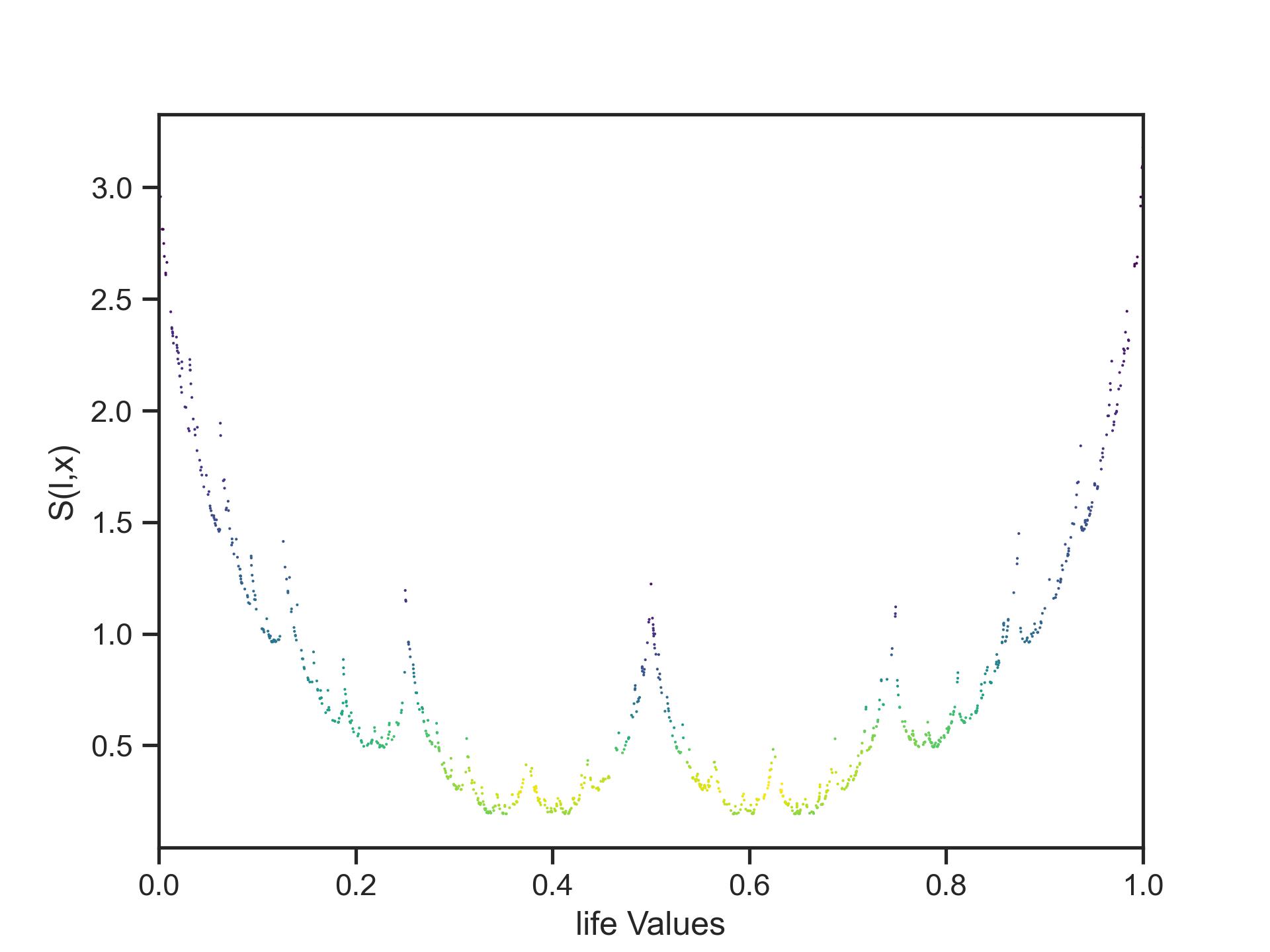}
        \caption{\begin{math}
            \gamma = 0.6
        \end{math}}
        \label{fig:plot2}
    \end{subfigure}
    
    \begin{subfigure}{0.45\textwidth}
        \includegraphics[width=\linewidth]{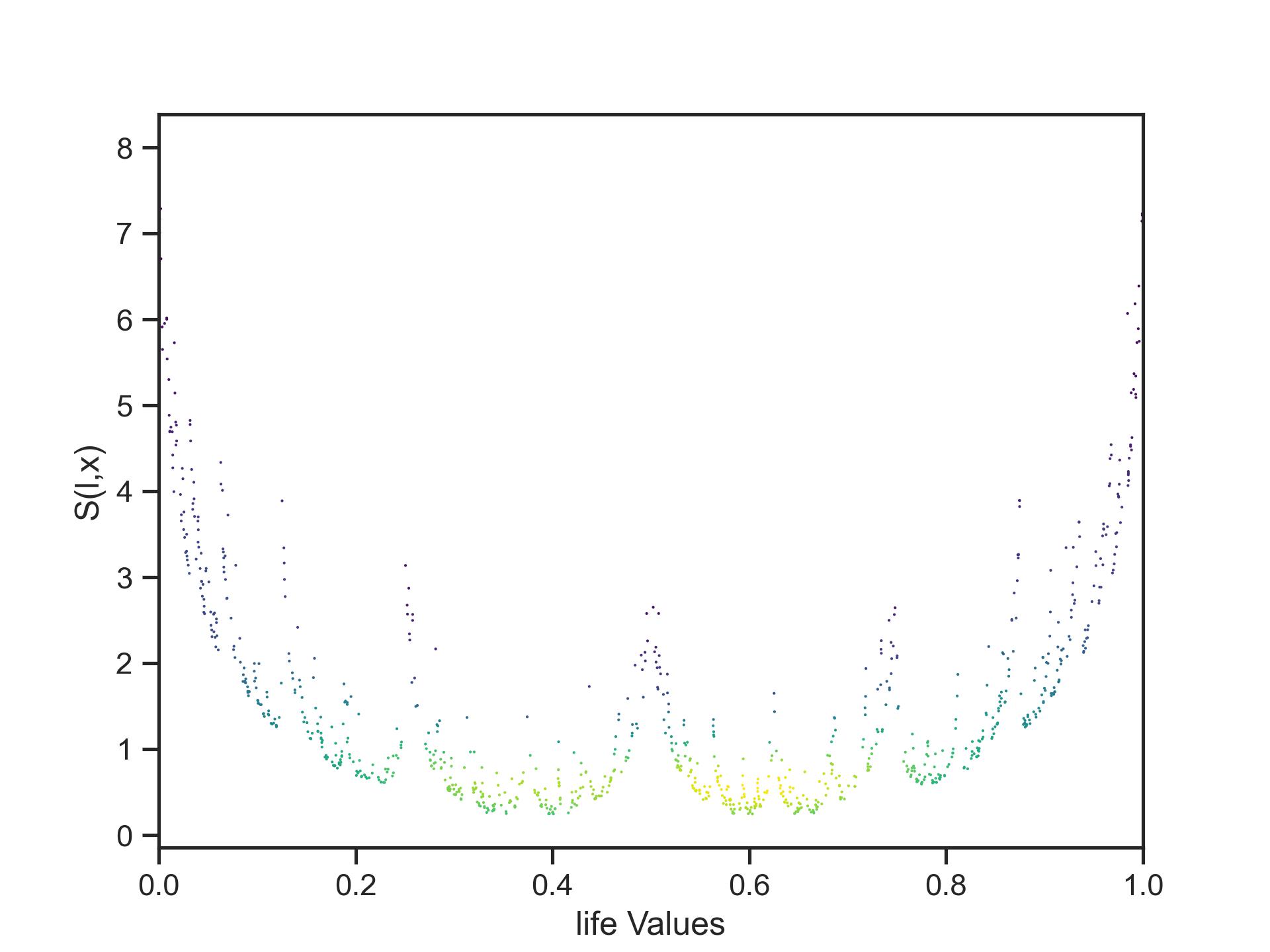}
        \caption{\begin{math}
            \gamma = 0.7
        \end{math}}
        \label{fig:plot3}
    \end{subfigure}
    \hfill
    \begin{subfigure}{0.45\textwidth}
        \includegraphics[width=\linewidth]{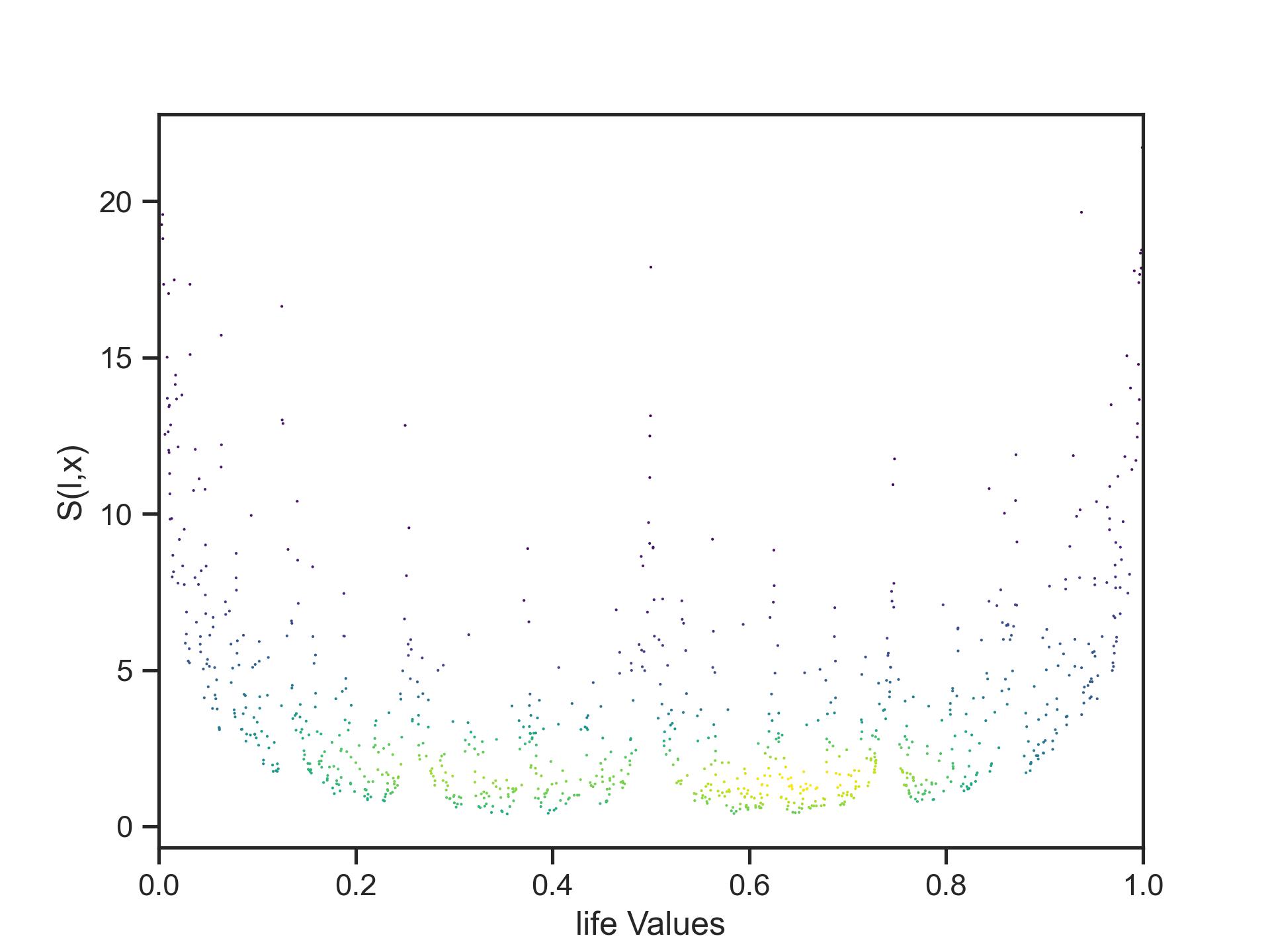}
        \caption{\begin{math}
            \gamma = 0.8
        \end{math}}
        \label{fig:plot4}
    \end{subfigure}
    
    \caption{Score-life function of the origin state of cart pole system for different values of \begin{math}
        \gamma 
    \end{math}}
    \label{fig:four_plots}
\end{figure}

\begin{figure}[htbp]
  \centering
  \begin{subfigure}[b]{0.45\textwidth}
    \includegraphics[width=\textwidth]{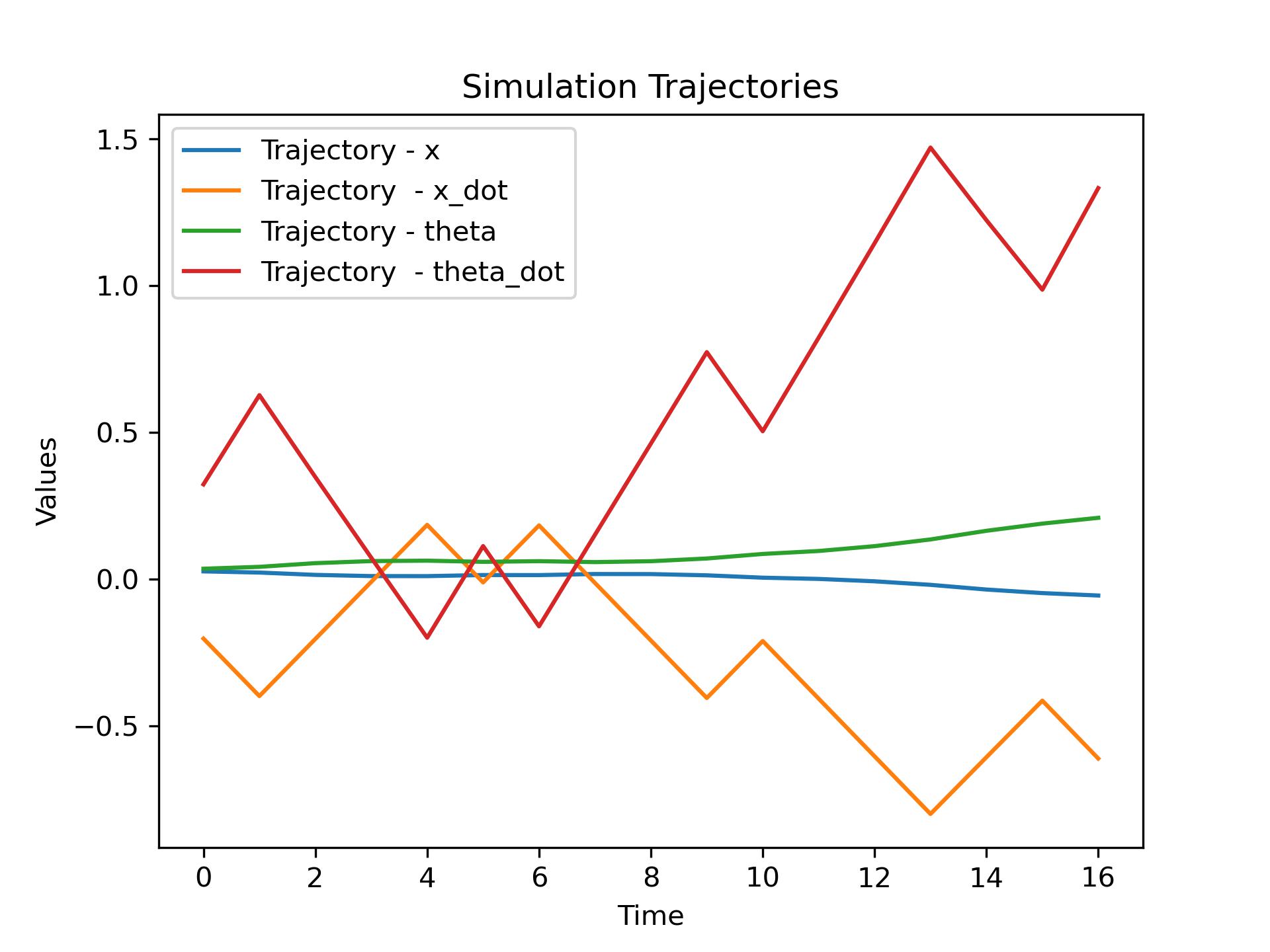}
    \caption{State Trajectory (Exact Methods)}
    \label{fig:image1}
  \end{subfigure}
  \hfill
  \begin{subfigure}[b]{0.45\textwidth}
    \includegraphics[width=\textwidth]{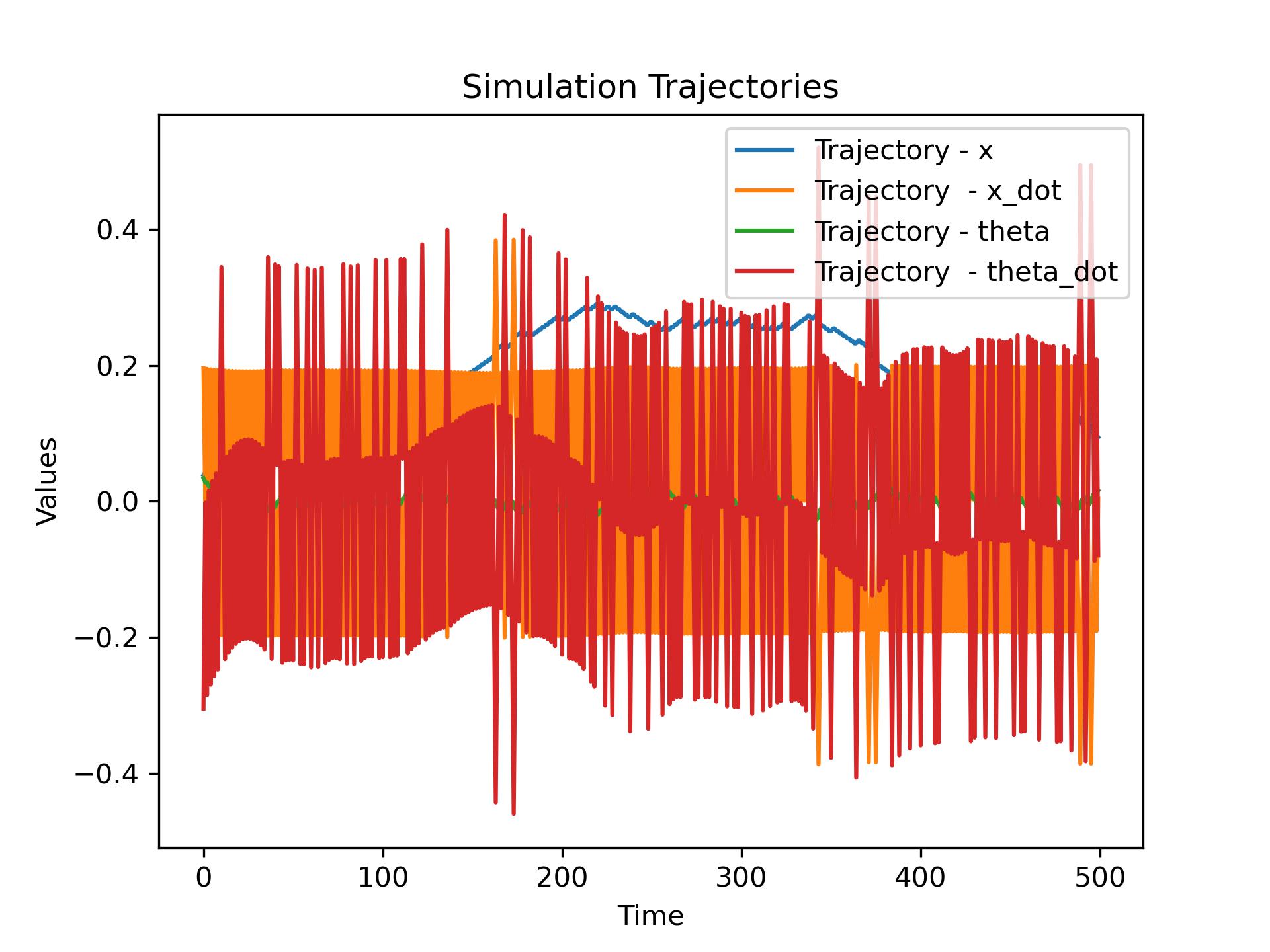}
    \caption{State Trajectory (Approximate Method)}
    \label{fig:image2}
  \end{subfigure}
  \caption{Trajectory plots comparing the exact and approximate methods for the Cart Pole Dynamical System.}
  \label{fig:side_by_side}
\end{figure}

\subsection{Exact Methods}

We applied Algorithm 2 for computing Faber Schauder representation of the Score-life function and computed optima of the fractal Score-life function using Algorithm 3. In our implementation, after computing \[
l^* = argmin_{l \in [0,1)} S(l,x)
\]
we extracted initial 10 bits of \begin{math}
    l^* 
\end{math} and applied the action sequence to the state \begin{math}
    x_0
\end{math}, \begin{math}
    x_1
\end{math}, until \begin{math}
    x_9
\end{math}, and after that we recomputed Faber Schauder representation of the Score-life function for state \begin{math}
    x_9
\end{math}, computed the optimal action sequence again and repeated the same scheme for all states. We noticed that the solutions from exact methods are not stable, this is likely due to errors in computing global optima of the fractal Score-life function. 
In principle, the optimal infinite horizon action sequence can be computed from the exact Score-life function, however, in practice, due to errors in the optimizer, we can only compute finite horizon action sequences. 

\subsection{Approximate Method}
We applied Algorithm 4 for computing the optimal infinite horizon cost and optimal actions for the cart pole dynamical system in openai gym \cite{brockman2016openai}. We used a quadratic approximation for the Score-life function, and we set \begin{math}
    \gamma = 0.8
\end{math}. 
\begin{figure}
\begin{center}
\includegraphics[scale=0.5]{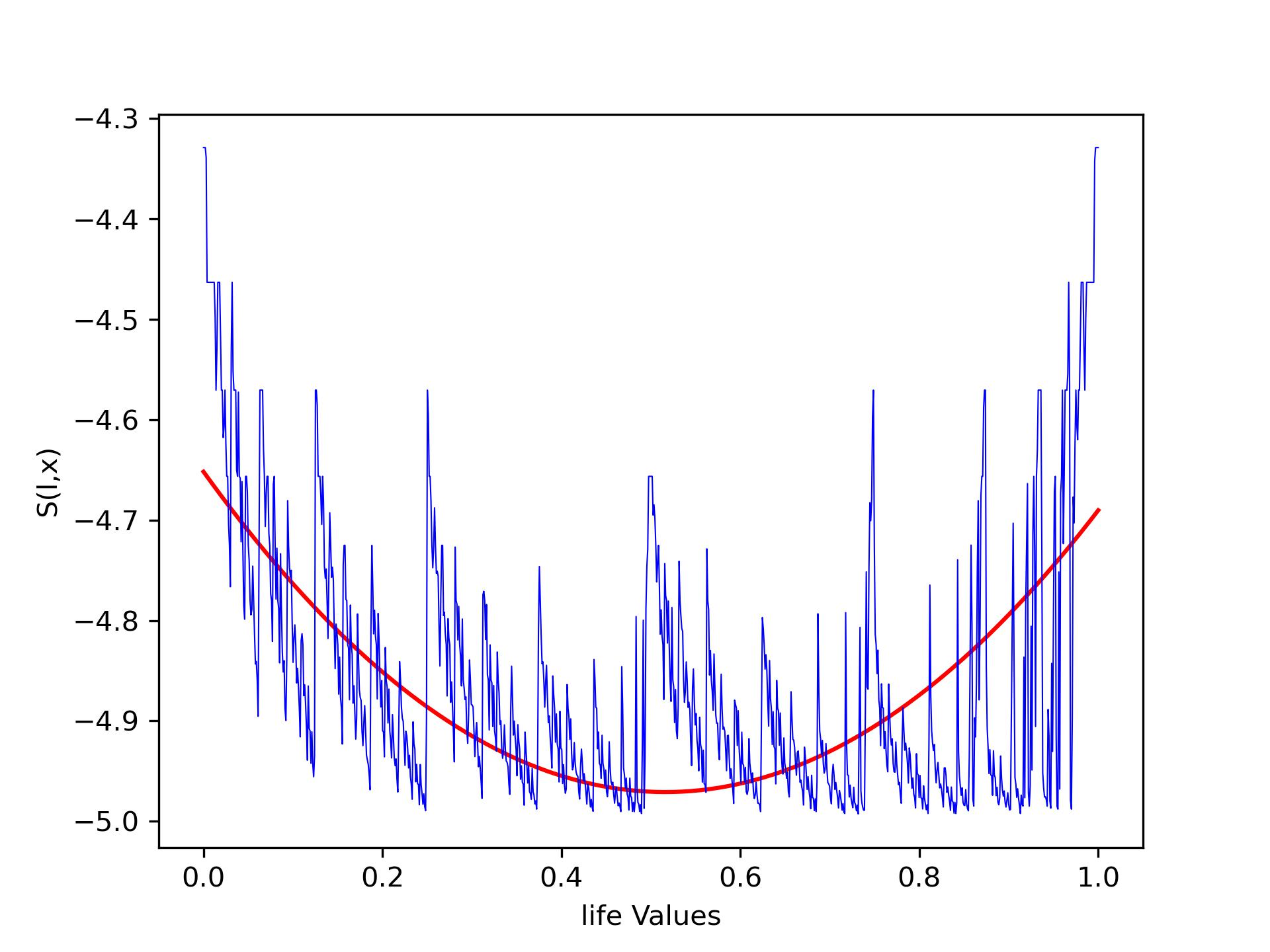}
\end{center}
\caption{ Approximate quadratic Score-life function shown in red and Exact Score-life function shown in blue}
\centering
\end{figure}
We used the predefined reward function in the cart pole environment, and set stage cost as negative of the reward value. Our method successfully balances the cartpole dynamical system for 500 timesteps, and achieves a cumulative reward value of 500. 

We noticed that in practice, approximate methods are more efficient and have much better performance in contrast to exact methods. This is likely due to the fact that in exact methods, the fractal Score-life function has many local minimas, and often the solver does not reach global optimal value. Hence, the computed action sequences from exact methods are suboptimal. This is not an issue in approximate methods, as in this case we use the polynomial approximation to compute optimal cost-to-go and compute instantaneous actions instead of action sequences. 
\end{document}